\documentclass[a4paper,11pt]{article}

\usepackage[top=3.5cm, bottom=3.5cm, left=2.5cm, right=2.5cm]{geometry}

\usepackage[utf8]{inputenc}
\usepackage[T1]{fontenc}
\usepackage[french,english]{babel} 
\usepackage{enumerate}
\usepackage{enumitem}
\usepackage{subcaption}
\usepackage{xcolor}
\usepackage[]{pdfpages}

\usepackage[normalem]{ulem}

\usepackage{titling}
\usepackage[colorlinks=false,linkcolor=blue]{hyperref}
\usepackage{csquotes}
\usepackage[nottoc, notlof, notlot, numbib]{tocbibind}
\usepackage[style=alphabetic,giveninits,doi=false,isbn=false,url=false,eprint=false,clearlang=true,sorting=nyt,maxbibnames=99]{biblatex}
\DeclareSourcemap{
      \maps[datatype=bibtex]{
      \map[overwrite=false]{
       \step[fieldsource=note]
       \step[fieldset=addendum, origfieldval, final]
       \step[fieldset=note, null]
       }
   }
   }

\usepackage{amsthm,amsfonts,amsmath,amssymb}
\usepackage{mathtools}
\usepackage{thmtools}
\usepackage{bbm}
\usepackage{mathrsfs}
\usepackage{float}
\usepackage{stmaryrd}
\usepackage{setspace}
\usepackage{array}
\usepackage{tabularx}
\usepackage{ltablex} 
\usepackage{comment} 

\usepackage{todonotes}
\mathtoolsset{showonlyrefs}

\newcommand{\R}{\mathbb{R}}

\newcommand{\E}{\mathbb{E}}
\newcommand{\ps}{\mathcal{P}} 

\renewcommand{\d}{\mathrm{d}}

\renewcommand{\P}{\mathbb{P}}

\renewcommand{\L}{\mathcal{L}}

\newcommand{\ndbar}{\overline{D}}

\newcommand{\dimin}{d_{\mathrm{in}}}
\newcommand{\dimout}{d_{\mathrm{out}}}

\newcommand{\diag}{\mathrm{diag}}
\newcommand{\loss}{\mathrm{loss}}

\newcommand{\hWin}[1]{\hat{W}^{\mathrm{in}#1}}
\newcommand{\hWout}[1]{\hat{W}^{\mathrm{out}#1}}

\newtheorem{theorem}{Theorem}[section]

\newtheorem{lemma}[theorem]{Lemma}

\newtheorem{assumption}{Assumption}

\theoremstyle{definition}

\newtheorem{example}[theorem]{Example}

\theoremstyle{remark}
\newtheorem{rem}{Remark}

\newcounter{paragraphe}

\title{Dropout and Random Gradient Masking Are Asymptotically Equivalent in Large ResNets}

\author{
Javier Maass\thanks{École Polytechnique Fédérale de Lausanne (EPFL),  Institute of Mathematics, 1015 Lausanne, Switzerland.},\quad
Lénaïc Chizat\thanks{Department of Mathematics,
National University of Singapore, Singapore.}
}

\date{\today}

\begin{document}

\maketitle

\begin{abstract}
Dropout and Random Gradient Masking (RaM) are two training techniques used to improve performance in deep learning. 
Both techniques inject  randomness into the training dynamics, but in significantly different ways: dropout applies random masks to the activations in the forward pass, whereas RaM leaves the forward pass unchanged and instead masks the gradients. In particular, the noise induced by RaM in the parameter updates is unbiased, so standard explanations for the effectiveness of dropout, such as the penalization effect or the prevention of co-adaptation between neurons, do not apply to RaM. In this work, we show that the difference between the two methods disappears for ResNets in the large depth and width asymptotics: in the complete feature learning regime, they both converge to the same large-scale limiting dynamics. This asymptotic equivalence holds for several variants of dropout and RaM, including layerwise dropout as used in stochastic-depth ResNets, albeit at slower quantitative rates. In fact, we also show that several of these variants collapse to the same limit asymptotically.
\end{abstract}

\section{Introduction}

Dropout is a deep learning technique, introduced
in~\cite{hinton2012improvingneuralnetworkspreventing}, which randomly deactivates 
subsets of units during training. It is easy to implement, largely architecture-agnostic,
and has been observed to improve test performance in a broad range of models, from
AlexNet~\cite{alexnet2012} to ResNets~\cite{he2015deepresiduallearningimage} and
Transformers~\cite{vaswani2017attention}. Several explanations have been proposed for
its effectiveness, including the prevention of co-adaptation between units
~\cite{hinton2012improvingneuralnetworkspreventing}, approximate model averaging
~\cite{baldi2013understanding}, penalization through the bias induced in the
gradient~\cite{mianjy2018implicit}, and implicit regularization induced by the variance
of the dropout noise~\cite{wei2020implicitexplicitregularizationeffects}.

\paragraph{Dropout in large-scale models.}
In order to disentangle these explanations,
~\cite{chizat2025phasediagramdropouttwolayer} proposed to decompose the effect of
dropout into three components: the \emph{propagation noise} induced in the forward and
backward passes; the \emph{penalization} coming from the bias induced by using the same
masks in the forward and backward passes; and the  \emph{random masking} of the
gradient update. For two-layer networks, they show that, depending on the joint scaling
of the width, dropout rate, and learning rate (LR), some or all of these effects
asymptotically disappear, leading to a rich phase diagram. In particular, in the standard
large-width limit with fixed dropout rate and fixed LR, they show that only the gradient
masking effect survives, making dropout equivalent to a \emph{Random Gradient Masking}
(RaM) method. Following this line of work, but moving toward practical architectures,
one may ask the following question:
\begin{center}
    \emph{What is the asymptotic behavior of dropout in large-scale ResNets?}
\end{center}

\paragraph{Contributions.}
In this paper, we show that, for ResNets, as for shallow neural networks, only the random
gradient masking effect persists in the large-scale regime. Specifically, we consider ResNets with
fixed dropout rate in the joint width-and-depth limit, and in the ``complete'' feature
learning regime~\cite{dey2026dontlazycompletepenables}, and show that the dropout and RaM training dynamics are asymptotically equivalent. Our framework covers a number
of dropout variants, characterized by their masking patterns, and we further show that
several of these variants collapse to the same dynamics. From a technical
perspective, we cannot rely on the mean-field viewpoint used
in~\cite{chizat2025phasediagramdropouttwolayer}, since it does not apply to practical
architectures beyond the two-layer case. Instead, our rigorous analysis is made possible
by the stochastic-approximation viewpoint on ResNet training dynamics recently
introduced in~\cite{chizat2025hidden}.

This equivalence challenges the conventional understanding of dropout's mechanisms in
deep architectures, showing in particular that both the propagation noise and the penalization effects vanish in the
limit.
It also suggests that RaM may serve as a useful proxy for dropout in sufficiently large residual architectures. In
fact, small-scale ablation experiments in~\cite{chizat2025phasediagramdropouttwolayer}
suggest that RaM alone performs comparably to dropout, while the recent
work~\cite{joo2026surprisingeffectivenessmaskingupdates} shows the ``surprising
effectiveness'' of RaM in large-scale settings. Going further, these results suggest revisiting
the theoretical explanation behind dropout to take into account its large-scale behavior:
namely, why does RaM help training? 

\paragraph{Organization.}
In Section~\ref{sec:setup}, we introduce the ResNet model and the training dynamics
with dropout or Random Gradient Masking (RaM). Section~\ref{sec:main-result}
presents the main asymptotic equivalence result (Theorem~\ref{thm:main}) and derives
the equations of the limiting model. Section~\ref{sec:numerics} provides empirical
validation of our theory, confirming the asymptotic equivalence of dropout and RaM in
deep ResNets. Finally, the proofs are deferred to the appendix.

\subsection{Other related work}
\paragraph{Dropout.} 
Ever since its introduction~\cite{hinton2012improvingneuralnetworkspreventing}, dropout has been part of standard training pipelines~\cite{vaswani2017attention,radford2018improving,devlin2019bertpretrainingdeepbidirectional,dosovitskiy2021imageworth16x16words,hu2021loralowrankadaptationlarge,ouyang2022traininglanguagemodelsfollow,ramesh2022hierarchicaltextconditionalimagegeneration,taylor2022galacticalargelanguagemodel}. Its empirical success sparked the introduction of many analogous techniques and variants, such as DropPath/StochasticDepth~\cite{huang2016deepnetworksstochasticdepth}, DropConnect~\cite{wan2013regularizationdropconnect}, Shake-Shake~\cite{gastaldi2017shakeshakeregularization}, 
ZoneOut for RNNs~\cite{krueger2017zoneoutregularizingrnnsrandomly}, SwapOut for ResNets~\cite{singh2016swapoutlearningensembledeep}, DropBlock for CNNs~\cite{ghiasi2018dropblockregularizationmethodconvolutional},
and StochasticPooling~\cite{zeiler2013stochasticpoolingregularizationdeep}, among many others. Our theoretical framework, detailed in Section~\ref{sec:setup}, allows us to cover a number of these variants. In modern large-scale settings, the use of dropout was advocated for in~\cite{xue2023repeatrepeatinsightsscaling,liu2023dropoutreducesunderfitting}, but also criticized in~\cite{liu2025drop}, yielding a contrasted situation.
On the theoretical side, many works have studied the dropout penalty, including its closed forms~\cite{arora2021explicitformscapacitycontrol} and its connections with classical regularization terms such as ridge regularization or path norms~\cite{wager2013dropouttrainingadaptiveregularization,arora2021explicitformscapacitycontrol,helmbold2015inductivebiasdropout,mou2018dropouttrainingdatadependentregularization}. The variance induced by dropout has also been studied for its effects on convergence properties and generalization~\cite{mianjy2020convergencegeneralizationdropouttraining,zhang2023implicitregularizationdropout,zhang2023stochasticmodifiedequationsdynamics,wei2020implicitexplicitregularizationeffects} but generally without distinguishing between the random masking and the propagation noise.

\paragraph{Random Gradient Masking (RaM).} 
The technique that we refer to as \emph{Random Masking} (RaM) is closely related to randomized block-coordinate descent~\cite{nesterov2012coordinatedescent,cotter2011betterminibatch,Wright2015CoordinateDA}. In the optimization literature, random coordinate descent was introduced to reduce the computational cost of each step. This motivation does not apply in deep networks, where the full forward and backward passes have to be computed anyway. Nonetheless, RaM has been observed to improve the performance of neural networks in both small~\cite{chizat2025phasediagramdropouttwolayer} and large-scale settings~\cite{joo2026surprisingeffectivenessmaskingupdates}. Since RaM uses \emph{unbiased} updates that depend on all units, it is unlikely that the explanations commonly invoked for the effectiveness of dropout (such as preventing co-adaptation or penalization effects) apply to it.

\paragraph{Large-scale limits of ResNets.}
Our analysis relies on recent progress in the theory of large-scale training dynamics. 
The type of limit model that we obtain is a Neural Mean ODE~\cite{lu2020meanfieldanalysisdeepresnet,ding2022mfresnets}, which lies at the crossroads of the Neural ODE framework~\cite{weinan2017proposal,chen2019neuralordinarydifferentialequations,lu2020meanfieldanalysisdeepresnet} and the mean-field analysis of 2LPs~\cite{nitanda2017stochastic,chizat2018globalconvergencegradientdescent,MeiMontanari2018mftwolayernetworks,sirignano2019meanfieldanalysisneural,Rotskoff2022mflimit}. The analysis of such models is an active line of research~\cite{jabir2019mean,isobe2023convergence,barboni2023globalconvergenceresnetsfinite,bonnet2023measure,barboni2025understandingtraininginfinitelydeep,daudin2025genericity,gassiat2025gradient}. The link between such models and practical architectures has long been unclear, until~\cite{chizat2025hidden} showed that this limit is relevant for standard ``shape'' regimes where, for instance, the widths of all layers are proportional and the depth diverges. The hyperparameter scalings that we consider, inherited from the Neural Mean ODE literature, are also consistent with the so-called  Complete Parameterization~\cite{dey2026dontlazycompletepenables}, which maximizes ``local'' feature updates and has been proposed as the appropriate generalization of the maximal update parameterization~\cite{yang2021tensorprogramsIV} to the large-depth setting. Since it is not needed to obtain our main result, we do not derive the large embedding-dimension limit of our dynamics, but we note that it would be possible by adapting~\cite{chaintron2026resnetsshapessizesconvergence}. See also~\cite{bordelon2024infinitelimitsmultiheadtransformer} for a prior derivation of the limit where all dimensions diverge that uses nonrigorous techniques from physics.

\section{Setup}\label{sec:setup}
\subsection{Training dynamics of ResNets with dropout}

\paragraph{Generic ResNets}
We consider a generic ResNet architecture, 
of depth $L$, hidden width $M$, and embedding dimension $D$, with $L,M,D\geq 1$. Let $\alpha >0$ be an additional scaling parameter, which can be taken equal to $1$ on first reading.
For a map $\phi: \R^p\times \R^D\to \R^D$, and a dropout mask $\zeta = {(\zeta^{\ell, j, d})}_{\ell,j,d} \in \R^{L\times M\times D}$ 
we define a \emph{masked} ResNet,
for $x\in \R^D$, as
\begin{equation}\label{eq:masked_resnet}
  \begin{cases}
    \hat{h}_\theta^{\zeta}(0, x) = x\\
    \hat{h}_\theta^{\zeta}(\ell, x) = \hat{h}_\theta^{\zeta}(\ell-1, x) + \frac{\alpha}{LM}\sum_{j=1}^M (1+ \zeta^{\ell, j})\odot \phi(z^{j,\ell}, \hat{h}_\theta^{\zeta}(\ell-1, x))
\end{cases}
\end{equation}
where we write $\zeta^{\ell, j} = {(\zeta^{\ell, j,d})}_{d=1}^D\in \R^D$ for the mask applied to each unit.
Here, $\theta = {(z^{j,\ell})}_{j=1,\ell =1}^{M, L} \in {(\R^{p})}^{M\times L}$ is the vector of parameters for the ResNet and $\odot$ denotes entrywise product of vectors. By setting $\zeta = 0$, we recover the vanilla ResNet architecture without dropout. 
In this paper, the mask $\zeta$ will be a centered random variable independently sampled at each training step, with potentially some dependency across its coordinates.

This description for ResNets is flexible and general enough to cover relevant cases in practice.

\begin{example}\label{ex:classic_examples}
 A ResNet with two-layer perceptron (2LP) blocks without intercepts can be obtained by letting $z = (u,v) \in \R^D \times \R^D$ (i.e.~$p = 2D$) and setting for $x \in \R^D$, \[\phi_{\mathrm{mlp}}((u,v),x) = v \rho(u^\top x),\] where $\rho : \R \to \R$ is the activation function.
 
    By instead setting \(\phi_{\mathrm{cnn}}((u,v),x) = v \star \rho(u \star x)\), where $\star$ denotes the discrete convolution operator, we obtain a convolutional ResNet~\cite{he2015deepresiduallearningimage}.
    ResNets with a single weight matrix per block are also covered by our analysis, by letting $M = 1$ and for instance $\phi(W,x) = W \rho(x)$ or $\phi(W,x) = \rho(W x)$ with $W \in \R^{D \times D}$.
We obtain the attention block of the transformer~\cite{vaswani2017attention} architecture by letting $z = (W_K, W_Q, W_V, W_O) \in (\R^{d_k \times D})^4$ and for an input family of $T$ tokens $x = (x_1, \dots, x_T) \in (\R^D)^T$,
\begin{equation}
    \phi_{\mathrm{att}}(z,x) = \left( W_O^\top \sum_{i=1}^T \frac{e^{(W_Q x_t)^\top (W_K x_i)/\sqrt{d_k}}}{\sum_{j=1}^T e^{(W_Q x_t)^\top (W_K x_j)/\sqrt{d_k}}} W_V x_i \right)_{1 \le t \le T} \in (\R^D)^T.
\end{equation}
In this setting, the hidden-width $M$ is known as the number of \emph{attention heads} per layer while $d_k$ is the key/query dimension, which is considered a constant in our analysis.
\end{example}

Consider a supervised learning problem with
 $\mathcal{X}= \mathcal{Y} = \R^D$ the input and output spaces, $\loss: \mathcal{Y}\times\mathcal{Y} \to \R$ a loss function and $\pi\in \ps(\mathcal{X}\times\mathcal{Y})$ a data distribution (which will be assumed of finite support later on). 
Although our proof technique applies to general gradient-based training, in order to fix ideas we consider gradient descent (GD) on the objective function $\L^{\zeta}(\theta) = \E_{\pi}[\mathrm{loss}(\hat{h}_\theta^{\zeta}(L, x), y)]$. The gradient reads
\begin{equation}\label{eq:gradient_masked_resnet}
  \nabla_{z^{j,\ell}} \L^{\zeta}(\theta) = \frac{\alpha}{LM} \E_{\pi}\big[ {D_1\phi(z^{j,\ell},\hat{h}_\theta^{\zeta}(\ell-1, x))}^{\top} \diag(1+ \zeta^{\ell, j})\hat{b}_\theta^{\zeta}(\ell, x, \nabla_1 \mathrm{loss}(\hat{h}_\theta^{\zeta}(L, x), y))\big], 
\end{equation}
where $\hat{b}_\theta^{\zeta}$ solves the backward pass equation for the dropout ResNet, given by, for inputs $x, w\in \R^D$:
\begin{equation}\label{eq:backward_masked_resnet}
  \begin{cases}
    \hat{b}_\theta^{\zeta}(L, x, w) = w,\\
    \hat{b}_\theta^{\zeta}(\ell-1, x, w) = \hat{b}_\theta^{\zeta}(\ell, x, w) + \frac{\alpha}{LM}\sum_{j=1}^M {D_2\phi(z^{j,\ell}, \hat{h}_\theta^{\zeta}(\ell-1, x))}^{\top} \diag(1+ \zeta^{\ell, j})\hat{b}_\theta^{\zeta}(\ell, x, w).
\end{cases}
\end{equation}

\paragraph{Dropout and variants}
We consider dropout masks that are resampled independently at each gradient-descent step. The mask vectors \(\zeta^{j,\ell} \in \mathbb{R}^D\) are identically distributed according to a common law \(\nu_\zeta \in \mathcal{P}(\mathbb{R}^D)\). Throughout, we assume that 
\begin{equation}\label{ass:dropout-distribution}
\text{
\(\nu_\zeta\) is centered and supported on
\(
\left\{x \in \mathbb{R}^D : \lVert x \rVert_\infty \leq B_\zeta\right\}
\)
for some \(B_\zeta > 0\).
}
\end{equation}

The standard choice used in practice is the centered and rescaled Bernoulli distribution with parameter \(q \in (0,1]\), where \(q\) denotes the \emph{keep rate}. Each coordinate then has the marginal distribution
\begin{equation}\label{eq:bernoulli_distrib}
  \zeta^{j,\ell,d}
  =
  \begin{cases}
    \dfrac{1-q}{q}, & \text{with probability } q,\\[4pt]
    -1, & \text{with probability } 1-q.
  \end{cases}
\end{equation}
The coordinates may either be sampled independently across \(d\), corresponding to coordinate dropout, or be shared across \(d\), corresponding to a common mask for the entire vector, which we refer to as \emph{unit dropout}.

As for the joint distribution of masks across units, we consider the following three variants:
\begin{itemize}
\item (independent masks) The $\zeta^{j,\ell}$ are independent across $j$ and $\ell$. 
\item (width-shared masks) The masks $\zeta^{j,\ell}$ are shared across $j$, and independent across $\ell$. That is $\zeta^{j,\ell}$ and $ \zeta^{j',\ell'}$ are equal if $\ell=\ell'$ and independent otherwise. This variant {(using \emph{unit dropout} masks)} is known as Stochastic Depth~\cite{huang2016deepnetworksstochasticdepth} as it turns off entire layers randomly.
\item (depth-shared masks)  The masks $\zeta^{j,\ell}$ are shared across $\ell$, and independent across $j$.
\end{itemize}
We always require that masks have at least independence across  either $L$ or $M$.

The original dropout method~\cite{srivastava2014dropout} applies the mask after the activation function. This is equivalent, in our terminology, to unit dropout with independent masks. To see this, e.g.~in the 2LP case, observe that for a mask $\zeta \in \R^{M}$ with i.i.d.~entries, it holds, for parameter matrices $U, V \in \R^{D\times M}$,
\[\frac{1}{M}V \big( (1+\zeta)\odot \rho (U^{\top} x)\big) = \frac{1}{M} \sum_{j=1}^M (1+\zeta^j) V^j\rho ({(U^j)}^\top x) = \frac{1}{M} \sum_{j=1}^M (1+\zeta^j) \phi_{\mathrm{mlp}}((U^j,V^j), x).\]

\paragraph{GD-Dropout dynamics}
The training dynamics using dropout is as follows. 
We write $\theta_k = {(\hat{Z}^{j,l}_k)}_{j,l}$ for the parameters at training iteration $k\geq 1$, switching to capital letters to indicate that these evolving parameters are now random variables. Fix a dropout variant, an initial probability distribution $\mu_0$ on $\mathbb{R}^p$ and a mask distribution $\nu_\zeta \in \P(\R^D)$ satisfying~\eqref{ass:dropout-distribution}. Initialize $\hat{Z}_{0}^{j,l} \sim \mu_0$ i.i.d.~and at each $k\geq0$ draw an independent sample of the dropout mask $\zeta_k \in \R^{M\times L\times D}$ and update the parameters as follows:
\[
\hat{Z}_{k+1}^{j,\ell} = \hat{Z}_{k}^{j,\ell} - \frac{LM}{\alpha^2}\tau \nabla_{z^{j,\ell}} \L^{\zeta_k}(\theta_k),
\]
where $\tau >0$ is the learning rate and the $\frac{LM}{\alpha^2}$ scaling ensures non-vanishing/exploding updates when $\alpha \gtrsim 1$.
Writing  
\(w^{\zeta}_{x,y} = \nabla_1 \mathrm{loss}(\hat{h}_{\theta_k}^{\zeta}(L, x), y)\) for simplicity, and writing the gradient explicitly, we obtain:
\begin{equation}\label{eq:parameter_update_resnet_dropout}
\hat{Z}_{k+1}^{j,\ell} = \hat{Z}_{k}^{j,\ell} - \frac{\tau}{\alpha}\E_{\pi}\big[\tilde{\mathfrak{u}}(\hat{Z}_{k}^{j,\ell}, \hat{h}_{\theta_k}^{\zeta}(\ell-1, x), \hat{b}_{\theta_k}^{\zeta}(\ell, x, w^{\zeta}_{x,y}))\big](1+ \zeta^{\ell, j}_k).
\end{equation}
where we define the update map $\tilde{\mathfrak{u}}: \R^{p}\times \R^D\times \R^D\to \R^{p\times D}$ given by $\tilde{\mathfrak{u}}(z, h, b) = {D_1\phi(z, h)}^{\top}\diag(b)$ and used the fact that $\diag(a)b = a \odot b = \diag(b)a$.
Note that if the mask is shared across $D$, we can replace $\tilde{\mathfrak{u}}$ by $\mathfrak{u}: \R^{p}\times \R^D\times \R^D\to \R^{p}$ given by $\mathfrak{u}(z, h, b) = {D_1\phi(z, h)}^{\top}b$, which is the same update map as in the vanilla case (see e.g.~\cite{chizat2025hidden} for details).

\paragraph{GD-RaM dynamics} By removing dropout in the forward and backward pass equations---but not in the update equation---we obtain another training dynamics, which we refer to as \emph{Random Gradient Masking} (RaM, as in~\cite{chizat2025phasediagramdropouttwolayer}). That is, 
\begin{equation}\label{eq:parameter_update_resnet_ram}
  \tilde{Z}_{k+1}^{j,\ell} = \tilde{Z}_{k}^{j,\ell} - \frac{\tau}{\alpha}\E_{\pi}\big[\tilde{\mathfrak{u}}(\tilde{Z}_{k}^{j,\ell}, \hat{h}_{\tilde{\theta}_k}^{0}(\ell-1, x), \hat{b}_{\tilde{\theta}_k}^{0}(\ell, x, \tilde{w}^{0}_{x,y}))\big](1+ \zeta^{\ell, j}_k),
\end{equation}
where, according to our notation, the $0$ superscript indicates that no mask is used for the forward/backward passes. All the variants of dropout described above can be turned into a RaM version by using the corresponding mask in the update rule.
As can be seen from its expression, for \emph{unit dropout}, the RaM update rule can be implemented by masking the gradient as~
\begin{equation}\label{eq:parameter_update_resnet_ram_implementable}
  \tilde{Z}_{k+1}^{j,\ell} = \tilde{Z}_{k}^{j,\ell} - \tau \frac{L M}{\alpha^2}(1+ \zeta^{\ell, j}_k)\nabla_{z^{j,\ell}} \L^{0}(\tilde{\theta}_k).
\end{equation}
This recovers a version of the standard random block coordinate descent algorithm, which has recently shown promising performance in large-scale settings~\cite{joo2026surprisingeffectivenessmaskingupdates}.

\section{Main result}\label{sec:main-result}

The main result of this work is that, in the infinite-depth limit $L\to \infty$ or in the limit $M,L\to \infty$, the difference between the GD-dropout and the GD-RaM dynamics vanishes. As discussed later in Section~\ref{sec:case_of_large_D}, our result also holds with diverging embedding dimension $D$ under appropriate conditions on the shape. We consider the following technical assumptions.
\begin{assumption}[Regularity assumptions]\label{ass:regularity_assumptions}
    There exists $B > 0$ such that:
\begin{enumerate}
    \item $\phi$ is $B$-Lipschitz, differentiable, its differential $D\phi$ is $B$-Lipschitz and $\|\phi(0,0)\|_2 \le B$;
    \item For some fixed dataset ${(x_i, y_i)}_{i=1}^n$ with $n\in \mathbb{N}$, $\pi = \frac{1}{n}\sum_{i=1}^n \delta_{(x_i, y_i)}$ , and the inputs satisfy $\max_i \|x_i\|_2 \le B$. For notational convenience, we will write \(\hat{h}_{k,i}^\ell := \hat{h}_{\theta_k}^{\zeta_k}(\ell, x_i)\), \(\tilde{h}_{k,i}^\ell := \hat{h}_{\tilde{\theta}_k}^{0}(\ell, x_i)\), and similarly for \(\hat{b}_{k,i}^\ell\), \(\tilde{b}_{k,i}^\ell\).
    \item The loss is differentiable in its first argument, such that $y\mapsto \nabla_1 \mathrm{loss}(y, y_i)$ is $B$-Lipschitz and $\|\nabla_1 \mathrm{loss}(0, y_i)\|_2 \le B$ $\forall i \in [n]$;
\end{enumerate}
\end{assumption}
The regularity assumed on $\phi$ is quite restrictive, since it does not cover the case of 2LP blocks (where $\phi$ and $D\phi$ are only locally Lipschitz). However, this allows us to simplify the proofs, and to focus on the main convergence mechanism that sheds light on dropout. Following~\cite[Section 4]{chizat2025hidden}, our analysis could be extended to standard 2LP blocks at the cost of more technical proofs involving additional tail controls.

\subsection{Asymptotic equivalence in the complete feature learning regime}
Let $\alpha=1$. Consider the GD-dropout dynamics $((\hat{Z}_k^{j,\ell})_k)_{j,\ell}$ defined in~\eqref{eq:parameter_update_resnet_dropout}, and the GD-RaM dynamics $((\tilde{Z}_k^{j,\ell})_k)_{j,\ell}$ defined in~\eqref{eq:parameter_update_resnet_ram}, with coupled randomness using the same random initialization and the same masks $(\zeta^{j,\ell})_{j,\ell}$ at each step.
Consider the distance between the two dynamics in parameter space, forward pass and backward pass respectively defined, with $s_\ell = \ell/L$ for $\ell \in [0:L]$, as
\[
    \Delta_k^Z := \max_{\substack{j \in [1:M] \\ \ell \in [1:L]}} \|\hat{Z}_k^{j,\ell} - \tilde{Z}_k^{j,\ell}\|, \quad
    \Delta_k^h := \max_{\substack{i \in [1:n] \\ \ell \in [0:L]}} \|\hat{h}_{k,i}^\ell - \tilde{h}_{k,i}^\ell\|, \quad
    \Delta_k^b := \max_{\substack{i \in [1:n] \\ \ell \in [0:L]}} \|\hat{b}_{k,i}^\ell - \tilde{b}_{k,i}^\ell\|.
\]
\begin{theorem}[Asymptotic Equivalence of Dropout and RaM]\label{thm:main}
Let Assumption~\ref{ass:regularity_assumptions} hold with $B>0$. Let $\mu_0\in\P(\R^p)$ be a subgaussian distribution with variance proxy $\sigma_0^2 \leq B$ and $\nu_\zeta \in \P(\R^D)$ be the mask distribution satisfying~\eqref{ass:dropout-distribution} with $B_\zeta >0$. Let
\begin{equation}\label{eq:beta}
\beta = \beta(M,L) \coloneqq 
\begin{cases}
1 & \text{if the masks are independent},\\
\sqrt{M} & \text{if the masks are width-shared},\\
\sqrt{L\log L} & \text{if the masks are depth-shared}.
\end{cases}
\end{equation}

Then, for all $k\geq 1$, there exist $c_1, c_2 > 0$ that only depend on $B$, $B_\zeta$, $D$ and $k\tau$, such that for all $\delta\in(0,1)$ with probability at least $1 - \delta$, it holds:
\[
  \max_{k'\leq k} \max \{\Delta_{k'}^Z, \Delta_{k'}^h, \Delta_{k'}^b\}
  \leq c_1 \left(\frac{1}{L}
    +  \beta\frac{1+\sqrt{\log(kn/\delta)}}{\sqrt{LM}}
  \right)
\]
provided that the right-hand side is smaller than $c_2$.
\end{theorem}
\begin{proof}
Consider the two random events of probability at least $1-\delta/2$ where the conclusion of Theorem~\ref{thm:cvg_to_limit} (presented in Section~
    \ref{subsec:limit} below) holds for Dropout, and for RaM, respectively. Then, by a triangle inequality, the conclusion of Theorem~\ref{thm:main} holds in the intersection of these two events, which has probability at least $1-\delta$.
\end{proof}
We can make the following remarks:
\begin{itemize}
    \item (Link with vanilla GD limit) When the masks are independent, the convergence rate is the same as the rate of convergence of vanilla ResNets towards their large-depth limit~\cite{chizat2025hidden}. However the constants a priori differ as there is a hidden dependency on $B_\zeta$ as well.
    \item (Keep rate scaling) In the case of Bernoulli masks~\eqref{eq:bernoulli_distrib}, the assumption of a fixed law $\nu_\zeta$ for the masks implies that the \textit{keep rate} $q\in(0,1]$ is constant and equal across all units. In particular, it is not scaled in terms of the shape parameters $L,M,D$. This is the common regime implemented in practice. As in~\cite{chizat2025phasediagramdropouttwolayer} for the case of 2LP, it would be interesting to study the case of scaled dropout rates, such as the sparse regime $q = \Theta(\frac{1}{LM})$, where we expect other effects in the limit.
    \item (Scaling of embedding dimension $D$) The above result does not track the dependency on $D$ in the bound, which is discussed in Section~\ref{sec:case_of_large_D}.
\end{itemize}

Our analysis also shows that all the dropout variants discussed above collapse to the same dynamics asymptotically (which a priori still depends on the choice of $\nu_\zeta$).

\begin{theorem}[Collapse of dropout variants]\label{thm:collapse_dropout_variants}
Under the assumptions of Theorem~\ref{thm:main}, consider GD-Dropout and GD-RaM with masks that are either (i) independent, (ii) width-shared or (iii) depth-shared. 
     Then, these dynamics all \emph{coincide asymptotically}, in the sense that the sequence of ResNet's outputs up to a fixed training horizon $k\geq 0$ all converge in probability to the same limit as $L\to \infty$ and $M\to \infty$ (and $\log L=o(\sqrt{M})$ for depth-shared).
\end{theorem}
\begin{proof}
    This follows directly from Theorem~\ref{thm:cvg_to_limit} presented in Section~
    \ref{subsec:limit} below.
\end{proof}

\subsection{Asymptotic equivalence in the lazy-ODE regime}\label{sec:equivalence_in_lazy_ode}
In the previous section, we assumed a scaling factor $\alpha=1$. Taking $\alpha=\alpha(M,L)$ diverging to $+\infty$ with $L$ and/or $M$ allows to cover other hyper-parameter regimes, where feature learning still exists but is not ``complete''~\cite{dey2026dontlazycompletepenables} meaning that the update of the parameters $(Z^{j,\ell})$ only contributes to a vanishing fraction of the update of the associated feature $\phi(h^{\ell-1}(x),Z^{j,\ell})$.  In this regime, the magnitude of the updates of the parameters scales as $O(\frac{1}{\alpha})$, see ~\cite{chizat2025hidden} for details in this context. 

Our asymptotic equivalence result still holds in such regimes. 
\begin{theorem}[Asymptotic Equivalence of Dropout and RaM in the \emph{lazy ODE regime}]\label{thm:lazy_thm}
Let Assumption~\ref{ass:regularity_assumptions} hold with $B>0$, let $\alpha \geq 1$ and let $\mu_0\in\mathcal{P}(\R^p)$ be a subgaussian distribution with variance proxy $\sigma_0^2 \leq B$.
Further assume that $\phi$ is twice differentiable with a $B$-Lipschitz cross differential $D_{2,1}\phi$, and that for all $h\in \R^D$ $\E_{\mu_0}[\phi(h, Z_0)] = \E_{\mu_0}[D_1\phi(h, Z_0)] = 0$. Consider a distribution $\nu_\zeta$ of dropout masks satisfying~\eqref{ass:dropout-distribution} with $B_\zeta>0$ and let $\beta$ be given by~\eqref{eq:beta} as a function of the dependency structure of the masks.

Then, for all $k\geq 1$, there exist $c_1, c_2 > 0$ that only depend on $B$, $D$, $k\tau$ and
$B_\zeta$, such that for all $\delta\in(0,1)$ with probability at least $1 - \delta$, it holds:
\[
  \max_{k'\leq k} \max \{\alpha\Delta_{k'}^Z, \Delta_{k'}^h, \Delta_{k'}^b\}
  \leq c_1 \left(\frac{1}{\alpha} + \frac{1}{L}
    +  \beta\frac{\alpha(1+ \sqrt{\log(kn/\delta)})}{\sqrt{LM}}
  \right)
\]
provided that the right-hand side is smaller than $c_2$.
\end{theorem}

Observe that this bound is relevant for the scalings for which the right-hand side vanishes. For instance, for independent masks $\beta=1$ and the bound vanishes if and only if $\alpha,L \to +\infty$ and $\alpha = o(\sqrt{LM})$. The proof is given in Appendix~\ref{app:proof_lazy_thm}.

Let us mention however that the equivalence may fail for other HP scalings. For instance, in the Neural Tangent Kernel regime~\cite{jacot2020neuraltangentkernelconvergence}, dropout already introduces nonvanishing randomness in the first output at initialization, while RaM does not; so the two dynamics cannot be equivalent in this regime. Similarly, in the ``SDE scaling'' $\alpha =\sqrt{ML}$, dropout induces nonvanishing randomness in the first forward pass, preventing an equivalence with RaM of the same form.

\subsection{The limit dynamics in the complete feature learning regime}\label{subsec:limit}

\paragraph{Limit model for vanilla-GD.}
Let $\alpha=1$ for the rest of this work.
In the case of vanilla ResNets, it is known (see e.g.~\cite{chizat2025hidden}) that when $L\to \infty$, with $M$ possibly jointly diverging as well, the training dynamics converge to that of the Mean ODE model, which is described as follows.
Let $\mu = {(\mu(\cdot|s))}_{s\in[0,1]} \subseteq \mathcal{P}(\R^p)$ be a family of measures on the parameter space with a suitable regularity in $s$. The Mean ODE model is the solution, for $x\in \R^D$, to
\begin{equation}\label{eq:mean_ode}
    h_\mu(0, x) = x, \qquad 
    \partial_s h_\mu(s, x) = \int \phi(z, h_\mu(s, x))\d \mu(z|s).
\end{equation} 
For computing gradients, we also consider the adjoint/backward equation, for $x,w\in \R^D$,
\begin{equation}\label{eq:backward_mean_ode}
    b_\mu(1, x,w) = w, \qquad 
    \partial_s b_\mu(s, x,w) = -\int {D_2\phi(z, h_\mu(s,x))}^{\top} b_\mu(s, x,w) \d \mu(z|s).
\end{equation}
The limit dynamics for vanilla-GD can be described as the evolution of a $\mathbb{R}^p$-valued stochastic process $(Z(s))_{s\in [0,1]}$ indexed by $s\in [0,1]$ as follows. Take $\xi_0 \sim \mu_0$ and initialize $Z_0(s) = \xi_0$ for all $s\in[0,1]$. Then, for $k\geq 0$, set
\[Z_{k+1}(s) = Z_k(s) - \tau \E_{\pi}[\mathfrak{u}(Z_k(s), h_{\mu_k}(s, x), b_{\mu_k}(s, x, \bar{w}_{x,y}))], \qquad \forall s\in [0,1]\]
where \(\bar{w}_{x,y} = \nabla_1 \mathrm{loss}(h_{\mu_k}(1, x), y)\) and where $\mu_k(\cdot|s) = \mathrm{Law}(Z_k(s))$ for all $s\in[0,1]$.

\paragraph{Limit model for GD-Dropout.} In the case of GD-dropout, we show that the limit dynamics is given by the same (deterministic) forward/backward system~\eqref{eq:mean_ode}-\eqref{eq:backward_mean_ode}, but with the infinite-dimensional version of the RaM update rule. The initialization $Z_0$ is the same as above.
Let $\zeta_0,\zeta_1,\dots \overset{iid}{\sim} \nu_\zeta$ 
and define recursively, for $s\in[0,1]$,
\begin{equation}\label{eq:limit_update_rule_ram}
  Z_{k+1}(s) = Z_k(s) - \tau \E_{\pi}[\tilde{\mathfrak{u}}(Z_k(s), h_{\mu_k}(s, x), b_{\mu_k}(s, x, \bar{w}_{x,y}))](1+\zeta_k).
\end{equation}
We note that the choice of dependencies of the mask across $s$ is arbitrary because $h_{\mu_k}$ only depends on the marginals $ \mathrm{Law}(Z_k(s))$. Here, we use the same mask $\zeta_k$ for all depths $s\in[0,1]$ as this simplifies some parts of the proof (namely, we can directly exploit the regularity of $s\mapsto Z_k(s)$, instead of that of $s\mapsto \mu_k(\cdot|s)$ which is less convenient to handle, see Appendix~\ref{app:properties_mean_ode}).

Consider the following coupling between GD-Dropout and the limit system (a similar coupling can be defined for GD-RaM).
Let $(\hat{Z}_k^{j,\ell})$ be the iterates of the GD-Dropout dynamics~\eqref{eq:parameter_update_resnet_dropout} with masks $(\zeta^{j,\ell}_k)_{k,j,\ell}$ and let $((Z_k^{j,\ell})_k)_{j,\ell}$ denote $L\times M$ copies of the limit dynamics~\eqref{eq:limit_update_rule_ram}, such that $Z_0^{j,\ell}(s) = \hat{Z}_0^{j,\ell}, \forall s \in [0,1]$; and such that the $(j,\ell)$-th copy uses the masks $(\zeta^{j,\ell}_k)_k$. Note that the copies of the limit $((Z_k^{j,\ell})_k)_{j,\ell}$ are not necessarily independent, but they have the same \emph{independence structure} as the masks $((\zeta^{j,\ell}_k)_k)_{j,\ell}$.\footnote{In the sense that, if $((\zeta^{j,\ell}_k)_k)_{j,\ell}$ are independent across $j$ (resp.\ across $\ell$), then so are $((Z_k^{j,\ell})_k)_{j,\ell}$.}
We define the distance between the two dynamics in parameter space, forward pass and backward pass respectively defined, as
\[
    \hat{\Delta}_k^Z := \max_{\substack{j \in [1:M] \\ \ell \in [1:L]}} \|\hat{Z}_k^{j,\ell} - Z_k^{j,\ell}(s_{\ell-1})\|, \quad
    \hat{\Delta}_k^h := \max_{\substack{i \in [1:n] \\ \ell \in [0:L]}} \|\hat{h}_{k,i}^\ell - h_{k,i}(s_\ell)\|, \quad
    \hat{\Delta}_k^b := \max_{\substack{i \in [1:n] \\ \ell \in [0:L]}} \|\hat{b}_{k,i}^\ell - b_{k,i}(s_\ell)\|.
\]
\begin{theorem}\label{thm:cvg_to_limit}
  Under the assumptions of Theorem~\ref{thm:main}, we have that
  for all $k\geq 1$, there exist $c_1, c_2 > 0$ that only depend on $B$, $D$, $k\tau$ and
$B_\zeta$, such that for all $\delta\in(0,1)$ with probability at least $1 - \delta$, it holds:
\(
  \max_{k'\leq k} \max \{\hat \Delta_{k'}^Z, \hat \Delta_{k'}^h, \hat \Delta_{k'}^b\}
  \leq c_1 \left(\frac{1}{L}
    +  \beta\frac{1+\sqrt{\log(kn/\delta)}}{\sqrt{LM}}
  \right)
\)
provided that the right-hand side is smaller than $c_2$, where the factor $\beta$ is given by~\eqref{eq:beta}.

Moreover, under the same coupling construction, the same result holds \emph{mutatis mutandis} when replacing GD-Dropout by GD-RaM~\eqref{eq:parameter_update_resnet_ram}.
\end{theorem}

The intuition behind the proof of Theorem~\ref{thm:cvg_to_limit} is the following. As in~\cite{chizat2025hidden}, we can interpret the forward pass (or the backward pass) through the ResNets as an ``Euler-Monte-Carlo'' discretization of the Mean ODE model that uses $M\times L$ ``approximate'' samples from the limit model (see Lemma~\ref{lem:stoch_approx}). In particular, in the large scale regime, the empirical averages across depth and width turn into expectations and, using the independence of the dropout masks, the variance they induce in the forward and backward passes asymptotically vanishes. Next, the subtlety in the proof consists in propagating this argument through the GD iterations, because the parameters are not independent anymore for $k\geq 1$, and the random masks used in the forward and backward pass are also not independent. These correlations are controlled via propagation of chaos arguments and high-probability concentration bounds.
The full proof is given in Appendix~\ref{app:stochastic_approximation}.

\subsection{The case of large embedding dimension (formally)}\label{sec:case_of_large_D}

In the previous sections, we tracked the depth $L$ and hidden width $M$ in the error bounds, but we did not track the dependencies in $D$. In typical architectures, $M$ and $D$ grow together, so it is important to understand whether our conclusions also hold as $D$ grows. 

In this section, we focus on the case of 2LP blocks (i.e. $\phi = \phi_{\mathrm{mlp}}$) and track the dependency of the bounds in $D$. For the sake of brevity, we do not provide complete proofs in this section and remain at a formal level. The results could be made rigorous following the analysis in \cite[Theorem 3]{chizat2025hidden}.

\paragraph{Embedded ResNets and their limit}
To allow for varying embedding dimension $D$, we add (potentially trainable) embedding/unembedding matrices $\hWin{}\in \R^{D\times \dimin}$, $\hWout{}\in \R^{D\times \dimout}$ to the ResNet architecture, initialized with a fixed variance independent of $L,M,D$. For the equivalence result to hold, we suppose that \emph{no dropout is applied in the first and last layers}. The forward pass reads, for an input $x\in \R^{\dimin}$:
\begin{equation}\label{eq:masked_resnet_with_embedding}
  \begin{cases}
    \hat{h}_\theta^{\zeta}(0, x) = \hWin{} x\\
    \hat{h}_\theta^{\zeta}(\ell, x) = \hat{h}_\theta^{\zeta}(\ell-1, x) + \frac{1}{LM}\sum_{j=1}^M (1+ \zeta^{\ell, j})\odot (V^{j,\ell}\rho((U^{j,\ell})\top \hat{h}_\theta^{\zeta}(\ell-1, x)/D))\\
    \hat{y}_\theta^{\zeta}(x) = \frac{1}{D} (\hWout{})^\top\hat{h}_\theta^{\zeta}(L, x)
\end{cases}
\end{equation}
where $\zeta = {(\zeta^{\ell, j, d})}_{\ell,j,d} \in \R^{L\times M\times D}$ is the \emph{(internal) dropout mask}.
In the spirit of our analysis, the limit \emph{dropout Mean ODE} model approached by GD-Dropout and GD-RaM is given by
\begin{equation}\label{eq:dropout_mean_ode_embedding}
h_\Theta(0, x) = \hWin{} x,\quad
    \partial_s h_\Theta(s, x) = \E\big[\rho(\langle U(s), h_\Theta(s, x) \rangle_{\ndbar})V(s)|\hWin{},\hWout{}],
    \end{equation}
    with \(y_\Theta(x) = \frac{1}{D} (\hWout{})^{\top}  h_\Theta(1, x)\), and parametrized by $\Theta = ((U(s), V(s))_{s\in[0,1]}, \hWin{}, \hWout{})$.

\begin{rem}
     In practice, dropout is applied also on the (un)embedding layers. In that case, a dropout mask $\eta = (\eta^{\mathrm{in}}, \eta^{\mathrm{out}})\in\R^{2D}$ has to be applied to the (un)embedding in the forward pass of the network, modifying~\eqref{eq:masked_resnet_with_embedding} as \(\hat{h}_\theta^{\zeta, \eta}(0, x) = (1+\eta^{\mathrm{in}})\odot\hWin{} x\) and \(\hat{y}_\theta^{\zeta, \eta}(x) = \frac{1}{D} (\hWout{})^\top((1+\eta^{\mathrm{out}})\odot\hat{h}_\theta^{\zeta, \eta}(L, x))\).
    These embedding masks $\eta^{\mathrm{in}/\mathrm{out}}$ are unaffected by the $M,L\to\infty$ asymptotics, and remain present in the limit model. Therefore, GD-Dropout and GD-RaM (as defined above) are not asymptotically equivalent. Instead, GD-Dropout is asymptotically equivalent to a modification of GD-RaM which uses the masks $\eta$ in the forward pass.
\end{rem}

\paragraph{Error bound with dependency in $D$.} Following the arguments presented in~\cite[Theorem 3]{chizat2025hidden},  we make the following formal claim which refines Theorem~\ref{thm:main} by including the $D$ dependency.

Assume that the initial parameters  $U$ and $V$ are iid, centered, with standard deviations $\sigma_U, \sigma_V = O(\sqrt{D})$ (and subgaussian variance proxies $O(D)$) and that their LRs $\tau$ (as in~\eqref{eq:parameter_update_resnet_dropout}) are multiplied by $D$. Assume that the activation function $\rho$ is smooth with $\rho'$ and $\rho''$ bounded, that the masks are bounded in $\ell_\infty$ by $B_\zeta = O(1)$, and that the embeddings are suitably normalized\footnote{Namely, for a constant $B=O(1)$, $\max_i\|\hWin{}x_i\|_{\mathrm{rms}} \leq B$ and $h\mapsto \hWout{}\nabla\mathrm{loss}(D^{-1}\hWout{\top}h, y_i)$ is $B$-Lipschitz in RMS norm for all $i\in[1:n]$. The analysis in~\cite{chizat2025hidden} also requires to introduce clipping functions in the dynamics. We do not describe them here as we remain at a formal level.}. Assume $\max\{\log(L), D\}\leq BM$. 

Then there exists $c_1, c_2>0$ independent of $D$, such that for $\delta>e^{-M}$, with probability at least $1-\delta$ (conditional on the potential randomness of $\hWin{}, \hWout{}$) it holds
\begin{equation}\label{eq:D-dependence}
  \max_{k'\leq k} \max_{\substack{i \in [1:n] \\ \ell \in [0:L]}} \frac{1}{\sqrt{D}}\| \hat{h}_{k',i}^{\ell} - h_{k',i}(s_\ell) \|_{2}
  \leq c_1 \left(\frac{1}{L}
    + \beta \frac{\sqrt{D} + \log(n/\delta)}{\sqrt{LM}}
  \right)
\end{equation}
provided that the right-hand side is smaller than $c_2$. 
The multiplier $\beta$ is defined by~\eqref{eq:beta}.
\begin{rem}
    Notably, when the masks are either depth-shared or width-shared, the bound~\eqref{eq:D-dependence} may not vanish for typical network shapes. For instance, when the mask is width-shared then $\beta=\sqrt{M}$ and the second term is of order $\Theta(\sqrt{D}/\sqrt{L})$, which rarely vanishes in practice.
\end{rem}

\section{Numerical experiments}\label{sec:numerics}
The code used to reproduce the experiments, implemented in JAX~\cite{jax2018github}, is available at
\url{https://github.com/xavimaass/dropout_and_ram_2026}.
The experiments were run on a single NVIDIA V100 GPU.

To illustrate Theorem~\ref{thm:main}, we train a ResNet with $\mathrm{tanh}$ activation, with fixed embedding/unembedding matrices and with the dropout masks being applied only on the internal layers, as described in~\eqref{eq:masked_resnet_with_embedding}.

We train ResNets of various shapes $L$ and $M$ (for simplicity, we fix $D=10$) on the binary classification task of recognizing digits $4$ and $7$ on the MNIST dataset. A subset of $10,000$ images is selected, and they are split into 80\% training and 20\% test sets. Images are flattened to vectors of dimension $\dimin = 784$ and the labels are one-hot-encoded with $\dimout = 2$.
We use SGD for the cross-entropy loss, with learning rate $\tau = 0.4$, batch size $64$, for $k = 200$ training iterations and using keep rate $q = 0.5$. We perform $5$ repetitions of each experiment. For ease of implementation via~\eqref{eq:parameter_update_resnet_ram_implementable}, we perform this experiment only for unit dropout as defined after~\eqref{eq:bernoulli_distrib}, and considering the \emph{independent}, \emph{width-shared} and \emph{depth-shared} mask variants. For all variants, the random initializations and masks for RaM and Dropout are coupled as in Theorem~\ref{thm:main}.

In Figure~\ref{fig:forward_error_mnist} we show the RMS distance between the forward passes, $\Delta^h_k$, for RaM and Dropout for the different masking strategies, as functions of $M$, $L$ and also the ``effective width'' $ML$ for independent masks. For this plot, we trained multiple networks with shapes $(D,L,M) = (10, 2^{j}, 2^{j'})$ for $j,j'\in[3:10]$ and we show the results after $k=50$ iterations and fixing specific values of $M$ and $L$. A more detailed visualization is provided in Appendix~\ref{app:additional_numerics}. We observe that these dynamics indeed become asymptotically close as the network size increases. In particular, for the fully independent masks we can clearly observe the $O(1/\sqrt{ML})$ expected rate of convergence predicted by Theorem~\ref{thm:main}.
In terms of $M$ and $L$ separately, as expected, we observe convergence at a rate of $O(1/\sqrt{M})$ for the \emph{depth-shared} variant and of $O(1/\sqrt{L})$ for the \emph{width-shared} variant. Accordingly, the distance in the \emph{depth-shared} variant stagnates as $L$ increases, and similarly for the \emph{width-shared} masks and $M$.

\begin{figure}
        \centering
        \begin{subfigure}{0.635\linewidth}
        \centering
        \includegraphics[width=\linewidth]{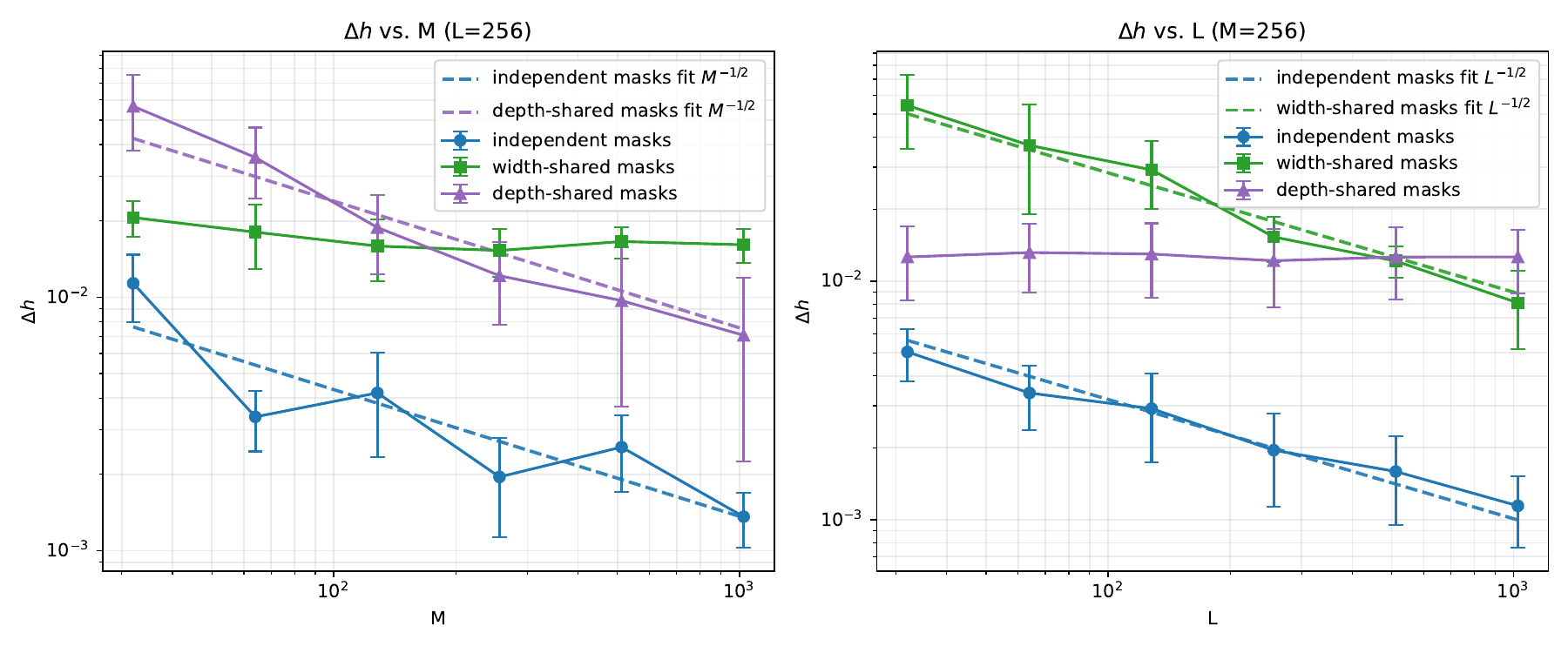}
        \caption{$\Delta^h_k$ against both $M$ and $L$ separately}
        \end{subfigure}%
        \begin{subfigure}{0.365\linewidth}
        \centering
        \includegraphics[width=\linewidth]{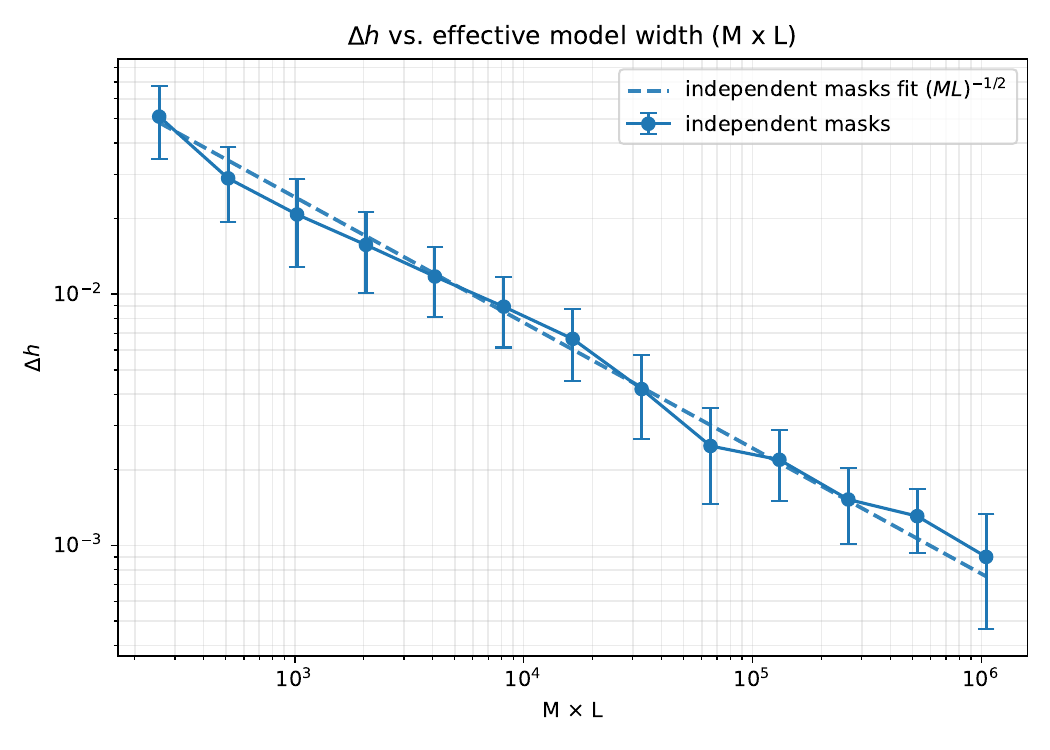}
        \caption{$\Delta^h_k$ vs $ML$}
        \end{subfigure}

        \caption{RMS error in the forward pass $\Delta^h_k$ between RaM and Dropout for the different masking strategies, as functions of $M$, $L$, and the ``effective width'' $ML$. Computed after $k=50$ GD iterations. We also fit the expected rates from Theorem~\ref{thm:main}. }
  \label{fig:forward_error_mnist}
\end{figure}

In Figure~\ref{fig:test_loss_mnist} we show that the behaviour of the test loss approximately coincides for SGD-Dropout and SGD-RaM for all the considered masking variants, as predicted by Theorems~\ref{thm:main} and~\ref{thm:collapse_dropout_variants}. They all also behave distinctly from vanilla-SGD. This behaviour becomes clearer as the network size increases in either axis: for large $M$ the \emph{depth-shared} variants collapse, and for large $L$ the \emph{width-shared} variants collapse.

\begin{figure}
        \centering
        \begin{subfigure}{0.5\linewidth}
        \centering
        \includegraphics[width=\linewidth]{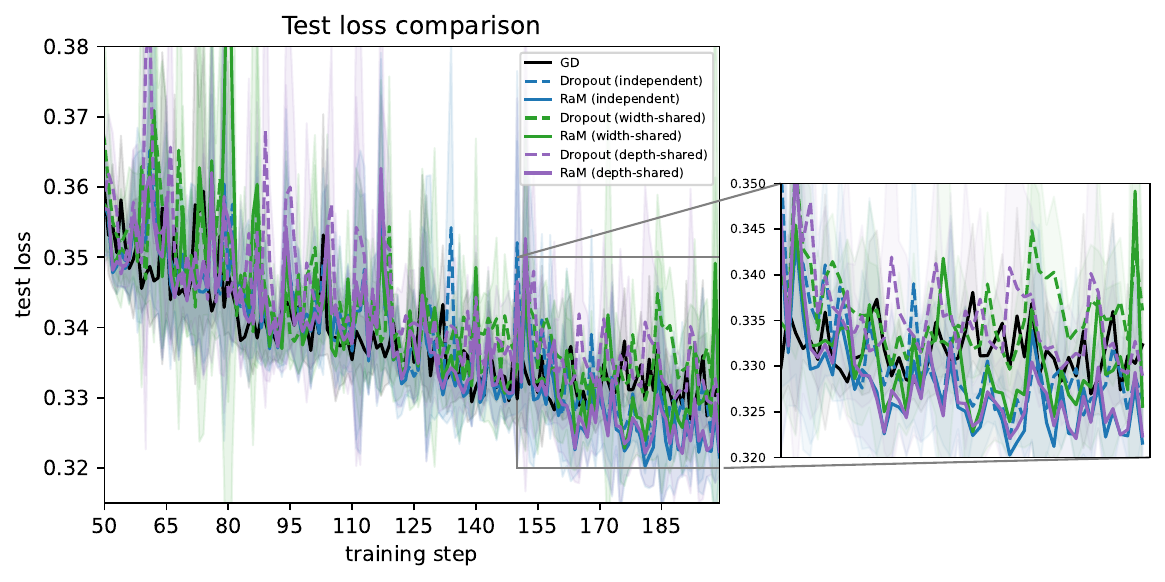}
        \caption{$L=M=64$}
        \end{subfigure}%
        \begin{subfigure}{0.5\linewidth}
        \centering
        \includegraphics[width=\textwidth]{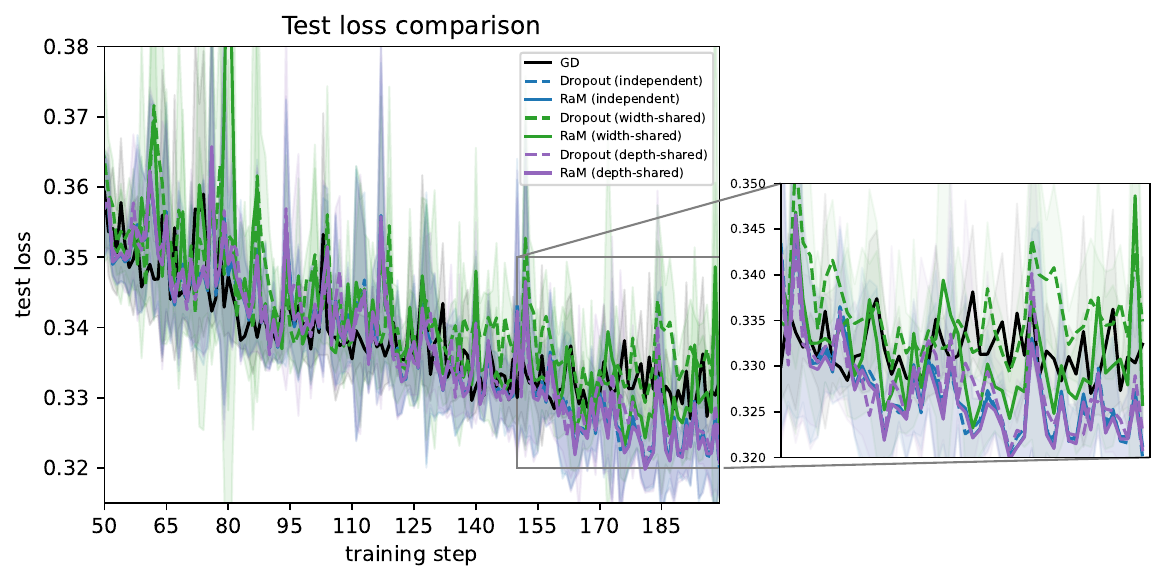}
        \caption{$L=64, M=1024$}
        \end{subfigure}\\
        \begin{subfigure}{0.5\linewidth}
        \centering
        \includegraphics[width=\linewidth]{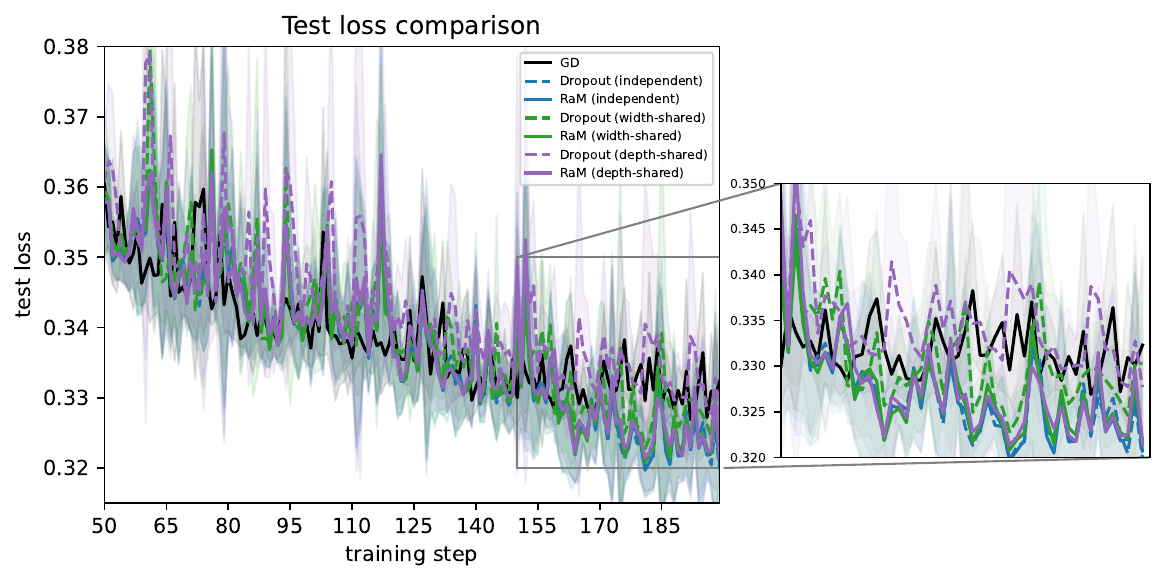}
        \caption{$L=1024, M=64$}
        \end{subfigure}%
        \begin{subfigure}{0.5\linewidth}
        \centering
        \includegraphics[width=\textwidth]{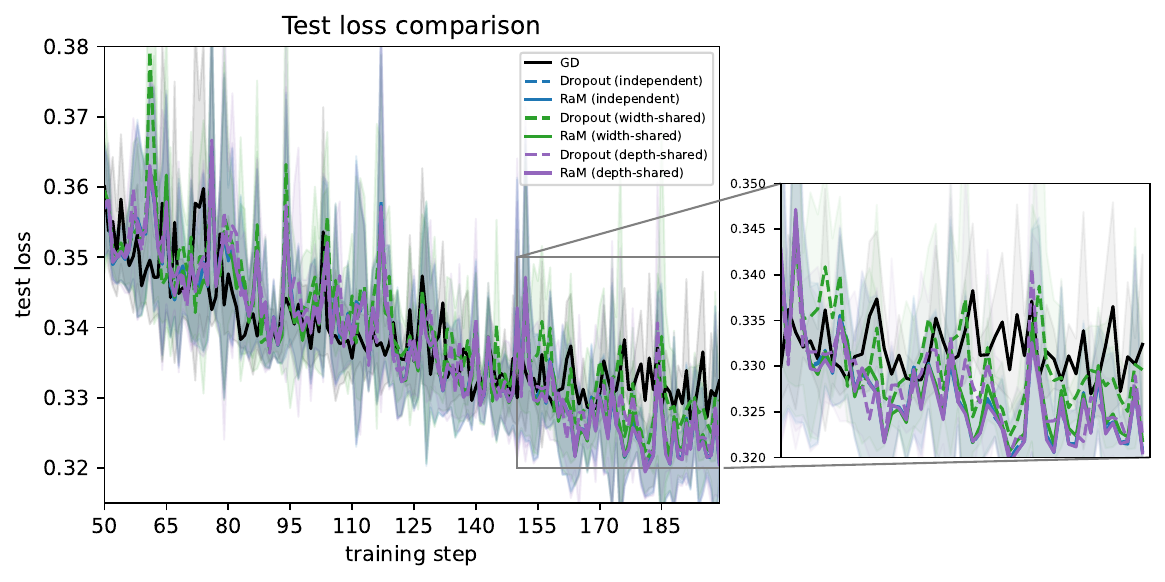}
        \caption{$L=M=1024$}
        \end{subfigure}\\
        
        \caption{Comparison of Dropout vs RaM vs vanilla SGD in test loss behaviour for the independent, width-shared and depth-shared masking variants, at different network sizes. The RaM and Dropout dynamics collapse, while SGD follows a distinct path.}
  \label{fig:test_loss_mnist}
\end{figure}

We stress that the goal of these experiments is mostly to illustrate the applicability and sharpness of our theory. A more extensive evaluation of the results in large-scale settings is a natural direction for future work.

\section{Conclusion}\label{sec:conclusion}

In this work, we have established the asymptotic equivalence of dropout and Random Gradient Masking (RaM) in the training dynamics of large-scale ResNets (Theorem~\ref{thm:main}). We have also shown that all dropout variants asymptotically behave identically (Theorem~\ref{thm:collapse_dropout_variants}). Our results rely on recent progress on the mathematical analysis of large ResNets.

This equivalence challenges the conventional understanding of dropout's mechanisms in
deep architectures, and our work leaves open a number of questions:
\begin{itemize}
\item In practical large-scale settings, can we replace dropout by RaM altogether? In fact, recent works~\cite{joo2026surprisingeffectivenessmaskingupdates, liu2025drop} imply that RaM compares favorably to dropout, suggesting that the finite-size effects of propagation noise and penalization of dropout could be undesirable.
\item Why does RaM improve the test performance? A partial explanation was proposed in~\cite{chizat2025phasediagramdropouttwolayer}, which showed that for wide two-layer neural networks, RaM makes it possible to increase the effective learning rate of active units while remaining in the stable region of GD. This property, which relies on the fact that the output of the model is an average of the contributions of the units, should also hold in the ResNet setting. However, this is far from a full answer to the question. In particular, it would be interesting to show a statistical benefit of RaM.
\end{itemize}

\newpage
\printbibliography

\newpage

\appendix

\section{Proofs of Theorem~\ref{thm:cvg_to_limit} and Theorem~\ref{thm:lazy_thm}}
We start this section by noting that the forward/backward equations for the \emph{dropout} ResNet~\eqref{eq:masked_resnet}-\eqref{eq:backward_masked_resnet} can be rewritten as:
\begin{equation}\label{eq:simplified_resnet_dropout}
    \begin{cases}
    \hat{h}_\theta^{\zeta}(\ell, x) = \hat{h}_\theta^{\zeta}(\ell-1, x) + \frac{1}{LM}\sum_{j=1}^M f^{h}(z^{j,l}, \hat{h}_\theta^{\zeta}(\ell-1, x){; \zeta^{\ell, j})}\\
    \hat{b}_\theta^{\zeta}(\ell-1, x, w) = \hat{b}_\theta^{\zeta}(\ell, x, w) + \frac{1}{LM}\sum_{j=1}^M f^{b}(z^{j,l}, \hat{h}_\theta^{\zeta}(\ell-1, x), \hat{b}_\theta^{\zeta}(\ell, x, w){; \zeta^{\ell, j})}
\end{cases}
\end{equation}
with $f^h: \R^p\times\R^D{\times\R^{D}} \to \R^{D}$ and $f^b: \R^p\times\R^D\times \R^D{\times\R^{D}} \to \R^{D}$ given by $f^h(z, h{; \xi}) = \phi(z, h) {\odot (\vec{1}+\xi)}$ and $f^b(z, h, b{; \xi}) = D_2\phi(z, h)^\top(b{\odot (\vec{1}+\xi)})$ respectively. {These are slight modifications of the functions appearing in~\cite{chizat2025hidden}.}

Analogously, the forward/backward Mean ODE~\eqref{eq:mean_ode}-\eqref{eq:backward_mean_ode} reads:
\begin{equation}\label{eq:simplified_mean_ode}
    \begin{cases}
    \partial_s h_\mu(s, x) = \int f^{h}(z, h_\mu(s, x){; \vec{0}}) \d\mu(z|s)\\
    \partial_s b_\mu(s, x, w) = -\int f^{b}(z, h_\mu(s, x), b_\mu(s, x, w){; \vec{0}}) \d\mu(z|s)
\end{cases}
\end{equation}
with the same functions $f^h$ and $f^b$.
Note that under Assumption~\ref{ass:regularity_assumptions}, these functions satisfy, {for any $\xi \in \R^{D}$,}
\[
\|f^h(z,h{;\xi}) - f^h(z',h'{;\xi})\|_2 = \|(\phi(z,h)-\phi(z',h')){\odot(1+\xi)}\|_2 \leq B{\|1+\xi\|_{\infty}}(\|z-z'\|_2 + \|h-h'\|_2),
\]
{where we used that $\|a\odot b\|_2 \leq \|a\|_2\|b\|_{\infty}$}. In particular, $f^h$ is $B{\|1+\xi\|_{\infty}}$-Lipschitz; and also
\begin{multline*}
    \|f^b(z,h, b{;\xi}) - f^b(z',h', b'{;\xi})\|_2 \\= \| D_2\phi(z, h)^{\top}((b- b'){\odot (1+\xi)}) + (D_2\phi(z, h) - D_2\phi(z', h'))^{\top} (b'{\odot(1+\xi)})\|_{2}\\ \leq B{\|1+\xi\|_{\infty}}\|b-b'\|_2 + B{\|1+\xi\|_{\infty}}\|b'\|_2(\|z-z'\|_2 + \|h-h'\|_2)
\end{multline*}
i.e. $f^b$ is $B{\|1+\xi\|_{\infty}}R$-Lipschitz whenever $b$ is restricted to a ball $\|b\|\leq R$.

Finally, it also holds that, for any centered random variable $\zeta$, $\E[f^h(z,h{;\zeta})] = f^h(z,h{;\vec{0}})$ and $\E[f^b(z,h, b{;\zeta})] = f^b(z,h, b{;\vec{0}})$.

\begin{rem}\label{rem:rewritten_update_map}
    {Note that the update equation for GD-Dropout in~\eqref{eq:parameter_update_resnet_dropout} can also be written as
\begin{equation*}
\hat{Z}_{k+1}^{j,\ell} = \hat{Z}_{k}^{j,\ell} - \frac{\tau}{\alpha}\E_{\pi}\big[\tilde{\mathfrak{u}}(\hat{Z}_{k}^{j,\ell}, \hat{h}_{\theta_k}^{\zeta}(\ell-1, x), \hat{b}_{\theta_k}^{\zeta}(\ell, x, w^{\zeta}_{x,y}); \zeta^{\ell, j}_k)\big],
\end{equation*}
where we abuse notation to write $\tilde{\mathfrak{u}}: \R^p\times\R^D\times\R^{D}\times\R^D \to \R^{p}$ for the function:
$\tilde{\mathfrak{u}}(z, h, b{; \xi}) = D_1\phi(z, h)^\top(b{\odot (\vec{1}+\xi)})$ (instead of $\tilde{\mathfrak{u}}(z, h, b) = D_1\phi(z, h)^\top\diag(b)$, as in the body of the paper). This will simplify the upcoming analysis and also applies to update rules~\eqref{eq:parameter_update_resnet_ram} and~\eqref{eq:limit_update_rule_ram}.
}

{
In particular, for fixed $\xi \in \R^D$, $\tilde{\mathfrak{u}}(\cdot; \xi)$ is $B{\|1+\xi\|_{\infty}}R$-Lipschitz whenever $b$ is restricted to a ball $\|b\|\leq R$. Indeed,
\begin{multline*}
    \|\tilde{\mathfrak{u}}(z_1, h_1, b_1;\xi)-\tilde{\mathfrak{u}}(z_2, h_2, b_2;\xi)\|_2 \leq \|{D_1\phi(z_1, h_1)}^{\top}(b_1\odot(1+\xi)) -{D_1\phi(z_2, h_2)}^{\top}(b_2\odot(1+\xi))\|_F\\
    \leq \|{D_1\phi(z_1, h_1)} -{D_1\phi(z_2, h_2)}\|_F\|b_1\odot(1+\xi)\|_2 + \|{D_1\phi(z_2, h_2)}\|_F\|(b_1-b_2)\odot(1+\xi)\|_2\\
    \leq B\|1+\xi\|_\infty\|b_1\|_2(\|z_1-z_2\| + \|h_1-h_2\|) + B \|1+\xi\|_{\infty}\|b_1-b_2\|,
\end{multline*}
from where the claim follows.
This computation also shows that $\tilde{\mathfrak{u}}(z, h, b) = D_1\phi(z, h)\diag(b)$ is such that $\|\tilde{\mathfrak{u}}(z, h, b)\|_{\infty\to2} \leq BR$; or, equivalently, $\xi \mapsto \tilde{\mathfrak{u}}(z, h, b;\xi)$ is $BR$-Lipschitz from $\Vert\cdot\Vert_\infty$ to $\Vert\cdot\Vert_2$.
}
\end{rem}
 
We will show that the noise from the forward and backward equations vanishes asymptotically.

\subsection{Properties of the Mean ODE RaM training dynamics}\label{app:properties_mean_ode}
We here restate some of the key lemmas from~\cite{chizat2025hidden} that will be relevant for our proof.

Consider a generic Mean ODE of the form
\begin{equation}\label{eq:generic_mean_ode}
    a(0) \in \R^D, \quad \dot{a}(s) = F(s, a(s)), \quad F(s,x) := \mathbb{E}[f(s,x,Z(s){;\vec{0}})]
\end{equation}
where $f : [0,1] \times \R^D \times \R^p{\times\R^{D}} \to \R^{D}$ and $(Z(s))_{s \in [0,1]}$ is an $\R^p$-valued stochastic process. The following standard regularity lemma is proved in~\cite[Lemma 5.1]{chizat2025hidden}.

\begin{lemma}[From~\cite{chizat2025hidden}]\label{lem:basic_regularity_mean_ode}
Let $\|\cdot\|$ be a norm on $\R^D$. Assume that $F$ is continuous in $s$ and $L_x$-Lipschitz in $x$ (in the $\|\cdot\|$ norm) for some $L_x \geq 1$. Let $B = \max_{s \in [0,1]} \|F(s,0)\|$ and $R := e^{L_x}(B + L_x\|a(0)\|)$. Then the mean ODE~\eqref{eq:generic_mean_ode} has a unique solution $a : [0,1] \to \R^d$ and it holds $\sup_{s \in [0,1]} \|a(s)\| \le R$ and $s \mapsto a(s)$ is $R$-Lipschitz continuous.
\end{lemma}

To make an analog of~\cite[Lemma 5.2]{chizat2025hidden}, we consider the following statement:
\begin{lemma}[Propagation of regularity]\label{lem:propagation_of_regularity}
Let Assumption~\ref{ass:regularity_assumptions} hold with $B > 0$ and consider the limit dynamics~\eqref{eq:limit_update_rule_ram} with $\mu_0 \in \mathcal{P}(\R^p)$ with first moment bounded by $B$ {and $\nu_\zeta \in \mathcal{P}(\R^D)$ bounded by $B_\zeta$}. Then, there exists $c$ that only depends on $B$ {and $B_\zeta$} such that for all $k \geq 0$, $i \in [1:n]$,

\begin{enumerate}
    \item $Z_k$, $h_{k,i}$ and $b_{k,i}$ are uniquely well-defined;
    \item the functions $s \mapsto h_{k,i}(s)$ and $s \mapsto b_{k,i}(s)$ are $c$-Lipschitz and bounded in norm by $c$;
    \item the function $s \mapsto Z_k(s)$ is $(e^{ck\tau} - 1)$-Lipschitz (surely).
\end{enumerate}
\end{lemma}
\begin{proof}
As in~\cite[Lemma 5.2]{chizat2025hidden}, we prove the three claims by recursion over $k$. 

\paragraph{Base case.} For $k = 0$, $Z_0$ is constant by definition. Then, for $i \in [1:n]$, $s \mapsto h_{0,i}(s)$ is well-defined and Lipschitz by Lemma~\ref{lem:basic_regularity_mean_ode} applied to the function $f_0^h(s,a,z{;\xi}) = \phi(z,a){\odot(1+\xi)}$ which is Lipschitz in $a$. In turn $s \mapsto b_{0,i}(s)$ is well-defined and Lipschitz by Lemma~\ref{lem:basic_regularity_mean_ode} applied to the function $f_0^b(s,a,z{;\xi}) = (D_2\phi(z,h_0(s, x_i)))^\top(a{\odot(1+\xi)})$ which is Lipschitz in $a$.

\paragraph{Inductive step.} By recursion, now assume that the claims hold at $k \in \mathbb{N}$ and that $s \mapsto Z_k(s)$ is $\Gamma_k$-Lipschitz. 
Following Remark~\ref{rem:rewritten_update_map}, we write
\begin{multline*}
    \|Z_{k+1}(s) - Z_{k+1}(s')\| \leq \|Z_k(s) - Z_k(s')\| \\+ \tau \max_{i \in [1:n]} \|\tilde{\mathfrak{u}}(Z_k(s), h_k(s, x_i), b_k(s, x_i){; \zeta_k}) - \tilde{\mathfrak{u}}(Z_k(s'), h_k(s', x_i), b_k(s', x_i){; \zeta_k})\|
\end{multline*}
{From Remark~\ref{rem:rewritten_update_map}, we know that $\tilde{\mathfrak{u}}$ is Lipschitz when restricted to $b$ in a ball, so that, using the inductive hypothesis, there exists $c > 0$ that depends only on $B$ (and that may change from line to line) such that}
\begin{multline*}
    {\|Z_{k+1}(s) - Z_{k+1}(s')\| \leq \Gamma_k|s-s'|} \\{+ c\tau\|1+\zeta_k\|_\infty (\Gamma_k|s-s'| + \max_{i \in [1:n]} \|h_k(s, x_i) - h_k(s', x_i)\| + \|b_k(s, x_i) - b_k(s', x_i)\|),}
\end{multline*}
{from where, using the boundedness of $\zeta_k$ by $B_\zeta$, we get
\(\|Z_{k+1}(s) - Z_{k+1}(s')\| \leq \Big(\Gamma_k + \tau c(\Gamma_k + 1)\Big) |s - s'|\).}
In particular, $s \mapsto Z_{k+1}(s)$ is $\Gamma_{k+1}$-Lipschitz with 
\[\Gamma_{k+1} \leq \Gamma_k(1 + c\tau) + \tau c.\] We can then apply Lemma~\ref{lem:basic_regularity_mean_ode} to obtain the well-posedness and Lipschitz regularity 
of $s \mapsto h_{k+1,i}(s)$ and of $s \mapsto b_{k+1,i}(s)$, in this order, $\forall i \in [1:n]$. With this, we conclude the recursion. 

\paragraph{Expression for $\Gamma_k$.}
The expression of $\Gamma_k$ (which does not play a role in what follows) comes from the following recursive argument: since $\Gamma_0 = 0$ and $\Gamma_{k+1} \leq \Gamma_k(1 + c\tau) + c\tau$ by the discrete Grönwall lemma\footnote{i.e. if $u_{k+1} \leq (1+\alpha)u_k + \beta$, then $u_k \leq e^{\alpha k} u_0 + \frac{\beta}{\alpha}(e^{\alpha k}-1)$. }, we get, since $\Gamma_0 = 0$,
$\Gamma_k \leq e^{ck\tau} - 1$.
\end{proof}

We then prove propagation of subgaussian tails as in Lemma~\cite[Lemma 5.3]{chizat2025hidden}. {For this, consider the following standard property of the subgaussian norm (see e.g.~\cite[Appendix A]{chizat2025hidden} for a proof).\footnote{{Recall that, for a random variable $Z\in \R$, its subgaussian norm is $\|Z\|_{\psi_2} = \inf\{K\geq0: \E[\exp(Z^2/K^2)] \leq 2\}$ and, for a random vector $X\in \R^p$, $\|X\|_{\psi_2} = \sup_{\|v\|_2=1} \|\langle v,X\rangle\|_{\psi_2}$. The \textit{variance proxy} seminorm is such that for absolute constants $c, C>0$, $c\|X-\E[X]\|_{\psi_2}\leq\|X\|_{vp}\leq C\|X-\E[X]\|_{\psi_2}$. For further properties of these objects and a general reference on the topic, see~\cite{vershynin2018highdimensionalproba}.}} 
\begin{lemma}\label{lem:subgaussian_via_lipschitz}
    If $X$ is a subgaussian vector in $\R^p$ and $f:\R^p\to \R^D$ is a $L$-Lipschitz function, then
    \[\|f(X)-\E[f(X)]\|_{\psi_2} \leq cL\sqrt{p}\|X - \E[X]\|_{\psi_2}, \]
    for an absolute constant $c>0$. Equivalently, \(\|f(X)\|_{vp} \leq cL\sqrt{p}\|X\|_{vp}\), where $\|\cdot\|_{vp}$ is the variance-proxy seminorm.
\end{lemma}
}
{We also consider the following fact: 
\begin{lemma}\label{lem:subgaussian_and_bounded_independent}
    If $X$ is a subgaussian vector in $\R^p$ and $Y$ is a random vector in $\R^{\tilde{p}}$ bounded by $B>0$ and independent of $X$. Let $f:\R^p\times\R^{\tilde{p}}\to \R^D$ be a (measurable) function. Then,
    \(\|f(X, Y)\|_{\psi_2} \leq \sup_{\|y\|\leq B}\|f(X,y)\|_{\psi_2}\). In terms of the variance-proxy\footnote{{Note that in general it is false that $\|f(X,Y)\|_{vp} \leq \sup_y \|f(X,y)\|_{vp}$ (unless $f(X,Y)$ is centered), as can be seen from choosing $f(x,y) = y$.}}, we have \[\|f(X, Y)\|_{vp}^2 \leq \sup_{\|y\|\leq B}\|f(X,y)\|_{vp}^2 + \frac{1}{4}\sup_{\|y\|, \|y'\|\leq B}\|\E[f(X,y)]-\E[f(X,y')]\|_{2}^2.\]
\end{lemma}
\begin{proof}
    {First, suppose $f$ is scalar valued.}Since $X$ and $Y$ are independent, for any measurable function $g$, $\E[g(X,Y)|Y] = \psi_g(Y)$ a.s. where $\psi_g(y) = \E[g(X,y)]$. In particular, let $c = \sup_{\|y\|\leq B}\|f(X,y)\|_{\psi_2}$, then
    \(\E\Big[\exp\Big(\frac{f(X, Y)^2}{c^2}\Big)\Big] = \E\Big[\E\Big[\exp\Big(\frac{f(X, Y)^2}{c^2}\Big)\Big|Y\Big]\Big] = \E[\psi(Y)]\), for the corresponding $\psi$. Now, by definition of $c$, for any $y$ with $\|y\|\leq B$, $\psi(y) = \E\Big[\exp\Big(\frac{f(X, y)^2}{c^2}\Big)\Big] \leq \E\Big[\exp\Big(\frac{f(X, y)^2}{\|f(X, y)\|_{\psi_2}^2}\Big)\Big] \leq 2$, where the last inequality comes from the definition of the subgaussian norm.
    Since $\|Y\| \leq B$ a.s. by definition, then a.s. $\psi(Y) \leq 2$, from where the claim follows. {For vector valued $f$, applying the result to $\langle v, f(X,Y)\rangle$ and then taking $\sup_{\|v\|_2 =1}$ concludes.}
    The claim for the variance proxy comes from an analogous argument, noting that the extra term comes from the fact that $\E[f(X, Y)|Y] - \E[f(X, Y)]$ is a centered and bounded (and thus subgaussian) random vector.
\end{proof}
}

\begin{lemma}[Propagation of subgaussian tails]\label{lem:propagation_of_subgaussianity}
Let Assumption~\ref{ass:regularity_assumptions} hold with $B > 0$ and consider the limit dynamics~\eqref{eq:limit_update_rule_ram} with $\mu_0 \in \mathcal{P}(\R^p)$ and $\nu_\zeta \in \P(\R^D)$. Assume moreover that $\mu_0$ is subgaussian {and $\nu_\zeta$ bounded by $B_\zeta$}. Then there exists $c > 0$ that only depends on $B$, $B_\zeta$ and $D$ such that $\forall k \geq 0$ and $s \in [0,1]$,
\[
    \|Z_k(s)\|_{vp} \le e^{ck\tau}({1+} \|Z_0\|_{vp}).
\]
\end{lemma}

\begin{proof}
By definition, $Z_0$ is subgaussian. Recursively, we use the seminorm properties of the variance-proxy to see that
\begin{multline*}
    \|Z_{k+1}(s)\|_{vp} \le \|Z_k(s)\|_{vp} + \tau \max_{i \in [1:n]} \|\tilde{\mathfrak{u}}(Z_k(s), h_{k,i}(s), b_{k,i}(s){;\zeta_k})\|_{vp}\\
    \le \|Z_k(s)\|_{vp} + \tau \max_{i \in [1:n]}\sup_{\|\xi\|_{\infty}\leq B_\zeta} \|\tilde{\mathfrak{u}}(Z_k(s), h_{k,i}(s), b_{k,i}(s){;\xi})\|_{vp} \\{+ c\sup_{\|\xi\|_{\infty}, \|\xi'\|_{\infty}\leq B_\zeta} \|\E[\tilde{\mathfrak{u}}(Z_k(s), h_{k,i}(s), b_{k,i}(s){;\xi}) - \tilde{\mathfrak{u}}(Z_k(s), h_{k,i}(s), b_{k,i}(s){;\xi'})]\|_{2},}
\end{multline*}
where in the last inequality we used Lemma~\ref{lem:subgaussian_and_bounded_independent} since $Z_k$ and $\zeta_k$ are independent.
{Whenever $b$ is restricted to a ball of radius $R$ (which it is by Lemma~\ref{lem:propagation_of_regularity}) and $\|\xi\|_\infty \leq B_\zeta$, the function $z \mapsto \tilde{\mathfrak{u}}(z,h,b;\xi)$ is $BR(1+B_\zeta)$-Lipschitz, so we can apply Lemma~\ref{lem:subgaussian_via_lipschitz}} to get
\[{\max_{i \in [1:n]}\sup_{\|\xi\|_{\infty}\leq B_\zeta} \|\tilde{\mathfrak{u}}(Z_k(s), h_{k,i}(s), b_{k,i}(s){;\xi})\|_{vp} \leq cBR(1+B_\zeta)\|Z_k\|_{vp}}.\]
Similarly, since $\xi \mapsto \tilde{\mathfrak{u}}(z,h,b;\xi)$ is $BR$-Lipschitz from $\Vert\cdot\Vert_\infty$ to $\Vert\cdot\Vert_2$, we have
\begin{multline*}
    \sup_{\|\xi\|_{\infty}, \|\xi'\|_{\infty}\leq B_\zeta} \|\E[\tilde{\mathfrak{u}}(Z_k(s), h_{k,i}(s), b_{k,i}(s){;\xi}) - \tilde{\mathfrak{u}}(Z_k(s), h_{k,i}(s), b_{k,i}(s){;\xi'})]\|_{2} \\\leq BR\sup_{\|\xi\|_{\infty}, \|\xi'\|_{\infty}\leq B_\zeta}\|\xi - \xi'\|_\infty \leq 2BRB_\zeta.
\end{multline*}
All in all, absorbing $B$, $B_\zeta$ and $R$ into the constant $c >0$, this yields
\[{\|Z_{k+1}(s)\|_{vp} \le (1+\tau c)\|Z_k(s)\|_{vp} + \tau c.}\]
{By the discrete Gronwall lemma, we conclude that there exists $c$ that only depends on $B$, $B_\zeta$ and $D$ such that for $k \ge 0$ it holds $\|Z_k(s)\|_{vp} \leq e^{c k \tau} (1+\|Z_0\|_{vp})$}.
\end{proof}

\subsection{Stochastic approximation lemma}\label{app:stochastic_approximation}

We consider a variant of the Stochastic approximation of Mean ODEs lemma from~\cite{chizat2025hidden}, in which the mask random variables are included. This version of the result considers strong regularity assumptions, which do not cover the case of 2LP blocks. A more technical version of this lemma under
weaker assumptions could be proposed following \cite[Lemma 6.4]{chizat2025hidden}.
\begin{lemma}[Stochastic approximation of Mean ODEs]\label{lem:stoch_approx}
Let {$f : [0,1] \times \R^D \times \R^p \times \R^{D}\to \R^{D}$} be deterministic and let
$Z : [0,1] \to \R^p$ be a random function. Assume that $F$ is continuous in $s$ and
that there exist $B$, $L_x$ and $\Gamma$ such that $x \mapsto \mathbb{E}[f(s,x,Z(s){;\vec{0}})]$
is $L_x$-Lipschitz and bounded by $B$ at $x = 0$ (uniformly in $s \in [0,1]$) and
$s \mapsto Z(s)$ is a.s.\ $\Gamma$-Lipschitz. Let $R > 0$ be given by Lemma~\ref{lem:basic_regularity_mean_ode}
(depending only on $B$, $\|a(0)\|_2$ and $L_x$) such that the unique solution of the
Mean ODE
\[
  a(0) \in \R^D, \quad
  a'(s) = F(s, a(s)), \quad
  F(s,x) := \mathbb{E}[f(s,x,Z(s){;\vec{0}})]
\]
satisfies $\sup_{s \in [0,1]} \|a(s)\|_2 \leq R$ and $s \mapsto a(s)$ is $R$-Lipschitz.
Moreover, we assume the following local controls:
\begin{itemize}
  \item there exists $\sigma_{R{,\zeta}} > 0$ such that for all $\|x\|_2 \leq R$, $s \in [0,1]$ {and $\|\xi\|_\infty \leq B_\zeta$}
        the random variable $f(s,x,Z(s){;\xi})$ is subgaussian with variance proxy
        $\sigma_{R{,\zeta}}^2$.
  \item there exists $L_{R{,\zeta}} > 0$ such that, {for all $\|\xi\|_{\infty}\leq B_\zeta$}, the function $f{(\cdot;\xi)}$ restricted to $\|x\|_2 \leq R$
        is $L_{R{,\zeta}}$-Lipschitz (jointly in all its variables).
\end{itemize}
For integers $M, L \geq 1$, and a random mask $(\zeta^{j,\ell})_{j,\ell}\subseteq\R^{{D}}$ {bounded by $B_\zeta$ (all identically distributed but possibly not independent across either $M$ or $L$)}
and independent from $Z$, let $s_\ell := \ell/L$ and consider the ``inexact masked Euler
Monte-Carlo'' scheme
\[
  \hat{a}_0 \in \R^D, \quad
  \hat{a}_\ell = \hat{a}_{\ell-1}
  + \frac{1}{LM} \sum_{j=1}^{M} \hat{f}\!\big(s_{\ell-1},\, \hat{a}_{\ell-1},\,
  \hat{Z}^{j,\ell}{; \zeta^{j,\ell}}\big), \quad \ell \in [1:L]
\]
where $(\hat{Z}^{j,\ell})_{j,\ell}$ are random vectors and, for some nonnegative random variables 
$\varepsilon_0, \varepsilon_1, \varepsilon_2 \geq 0$, the discrete model satisfies, 
\begin{enumerate}
  \item \label{lem:stoch_approx_assi} Bounded initial mismatch: $\|\hat{a}_0 - a(0)\|_2 \leq \varepsilon_0$.
  \item \label{lem:stoch_approx_assii} Bounded model error:
        \[{\sup_{\|\xi\|_{\infty}\leq B_\zeta }}\sup_{z \in \R^p} \sup_{\|x\|_2 \leq 2R} \sup_{\ell \in [1:L]}
        \|\hat{f}(s_\ell, x, z{;\xi}) - f(s_\ell, x, z; {\xi})\|_2 \leq \varepsilon_1.\]
  \item \label{lem:stoch_approx_assiii}Approximately independent parameters: there exists a family of (potentially dependent)
        samples $Z^{j,\ell}$ of $Z$ for $j \in [1:M]$, $\ell \in [1:L]$ such that
        $\|\hat{Z}^{j,\ell} - Z^{j,\ell}(s_{\ell-1})\|_2 \leq \varepsilon_2$ a.s. These samples have the same dependency structure as the masks, i.e. if $(\zeta^{j,\ell})_{j,\ell}$ are independent across $M$, then so are $(Z^{j,\ell})_{j,\ell}$.
\end{enumerate}
We further suppose that the model is centered, that is for all $(s,x,z) \in [0,1]\times\R^D\times\R^p$ $\E[f(s,x,z;\zeta^{j,\ell})] = f(s,x,z;\vec{0})$; and that $\sup_{\|\xi\|,\|\xi'\|\leq B_\zeta}\|\E[f(s,x,Z(s);\xi)] -\E[f(s,x,Z(s);\xi')]\|_2 \leq R_{\zeta}$ for a constant $R_{\zeta}$.\footnote{Alternatively, it is sufficient to assume that $\xi \mapsto f(\cdot; \xi)$ is $R_{\zeta}$-Lipschitz from $\|\cdot\|_\infty$ to $\|\cdot\|_2$.}
Then there exist $c_1, c_2 > 0$ that only depend on $B$, $L_x$, $L_{{R,\zeta}}$, $\Gamma$,
$\|a(0)\|_2$ {$R_{\zeta}$ and $B_\zeta$} such that with probability at least $1 - \delta$, it holds
\[
  \sup_{\ell \in [1:L]} \|\hat{a}_\ell - a(s_\ell)\|_2
  \leq {c_1} \left(
    \varepsilon_0 + \varepsilon_1 + \varepsilon_2 
    + \frac{1}{L}
    +  {\beta}\frac{{(1+\sigma_{R,\zeta})}(\sqrt{D} + \sqrt{\log(1/\delta)} )}{\sqrt{LM}}
  \right)
\]
provided that the right-hand side is smaller than $c_2$. {The multiplier $\beta$ encodes the dependency structure of $(\zeta^{j,\ell}, Z^{j,\ell})_{j,\ell}$. Namely, we set $\beta =1$ when this random family is fully independent, $\beta = \sqrt{M}$ when they are only independent across $[1:L]$ and $\beta=\sqrt{L\log(L)}$ when they are only independent across $[1:M]$.}
\end{lemma}

\begin{proof}[Proof of Lemma~\ref{lem:stoch_approx}]
We closely follow the proof of \cite[Lemma 5.4]{chizat2025hidden}.

For $\ell \in [0:L-1]$, it holds
\begin{align*}
a(s_{\ell+1}) - \hat{a}_{\ell+1}
&= a(s_\ell) - \hat{a}_\ell
  + \int_{s_\ell}^{s_{\ell+1}} \dot{a}(s)\, ds
  - \frac{1}{ML} \sum_{j=1}^{M} \hat{f}(s_\ell, \hat{a}_\ell, \hat{Z}^{j,\ell+1}{;\zeta^{j,\ell+1}}) \\
&= a(s_\ell) - \hat{a}_\ell
  + \underbrace{
      \int_{s_\ell}^{s_{\ell+1}} \dot{a}(s)\, ds - \frac{1}{L} F(s_\ell, a(s_\ell))
    }_{e^{\ell+1}_{\mathrm{euler}}} \\
&\quad + \underbrace{
      \frac{1}{L} F(s_\ell, a(s_\ell))
      - \frac{1}{ML} \sum_{j=1}^{M} f(s_\ell, a(s_\ell), Z^{j,\ell+1}(s_\ell){;\zeta^{j,\ell+1}})
    }_{e^{\ell+1}_{\mathrm{mc}}} \\
&\quad + \underbrace{
      \frac{1}{ML} \sum_{j=1}^{M}
      \left[ f(s_\ell, a(s_\ell), Z^{j,\ell+1}(s_\ell){;\zeta^{j,\ell+1}})
      - \hat{f}(s_\ell, \hat{a}_\ell, \hat{Z}^{j,\ell+1}{;\zeta^{j,\ell+1}}) \right]
    }_{e^{\ell+1}_{\mathrm{approx}}}.
\end{align*}
By recursion, we have
\[
  a(s_\ell) - \hat{a}_\ell
  = a(0) - \hat{a}_0
  + \sum_{k=1}^{\ell} e^k_{\mathrm{euler}}
  + \sum_{k=1}^{\ell} e^k_{\mathrm{mc}}
  + \sum_{k=1}^{\ell} e^k_{\mathrm{approx}}
\]
and therefore, with $\Delta^a_\ell := \|a(s_\ell) - \hat{a}_\ell\|_2$, it holds
\begin{equation}\label{eq:error_control_stoch_approx}
    \Delta^a_\ell
  \leq \|a(0) - \hat{a}_0\|_2
  + \sum_{k=1}^{\ell} \|e^k_{\mathrm{euler}}\|_2
  + \sum_{k=1}^{\ell} \|e^k_{\mathrm{approx}}\|_2
  + \left\| \sum_{k=1}^{\ell} e^k_{\mathrm{mc}} \right\|_2
\end{equation}
Note that for the Monte-Carlo error term, we take the norm after summing across layers.
We now bound these error terms one by one. For the ``Euler error'' term, we use the Lipschitz continuity of $f$,
$a$ and $Z$, and so it holds for $\ell \in [0:L-1]$
\begin{align*}
\|e^{\ell+1}_{\mathrm{euler}}\|_2
&= \left\| \int_{s_\ell}^{s_{\ell+1}}
   \bigl( F(s, a(s)) - F(s_\ell, a(s_\ell)) \bigr)\, ds \right\|_2 \\
&\leq \int_{s_\ell}^{s_{\ell+1}}
  \mathbb{E}\bigl[\|f(s, a(s), Z(s){;\vec{0}}) - f(s_\ell, a(s_\ell), Z(s_\ell){;\vec{0}})\|_2\bigr]\, ds \\
&\leq L_{R{, \zeta}}(1 + R + \Gamma)
  \int_{s_\ell}^{s_{\ell+1}} |s - s_\ell|\, ds
  \leq \frac{c}{L^2}
\end{align*}
for some $c > 0$ that only depends on $R$, $L_{R{, \zeta}}$ and $\Gamma$. 

For the ``approximation error'', under the assumption that
$\Delta^a_{\ell-1} \leq R$ so that $\|a(s_{\ell-1})\|_2, \|\hat{a}_{\ell-1}\|_2 \leq 2R$, we use the
regularity of $f$ and Assumptions \ref{lem:stoch_approx_assii} and \ref{lem:stoch_approx_assiii} to obtain
\begin{align}\label{eq:stoch_approx_bound_err_approx}
\|e^{\ell+1}_{\mathrm{approx}}\|_2
&= \left\| \frac{1}{ML} \sum_{j=1}^{M}
      \left[ f(s_\ell, a(s_\ell), Z^{j,\ell+1}(s_\ell){;\zeta^{j,\ell+1}})
      - \hat{f}(s_\ell, \hat{a}_\ell, \hat{Z}^{j,\ell+1}{;\zeta^{j,\ell+1}}) \right]\right\|_2 \notag\\
&\leq \frac{1}{ML} \sum_{j=1}^{M}
      \left\|f(s_\ell, a(s_\ell), Z^{j,\ell+1}(s_\ell){;\zeta^{j,\ell+1}})
      - f(s_\ell, \hat{a}_\ell, \hat{Z}^{j,\ell+1}{;\zeta^{j,\ell+1}}) \right\|_{2} \notag\\
      &\qquad +\frac{1}{ML} \sum_{j=1}^{M}
      \left\|f(s_\ell, \hat{a}_\ell, \hat{Z}^{j,\ell+1}{;\zeta^{j,\ell+1}})
      - \hat{f}(s_\ell, \hat{a}_\ell, \hat{Z}^{j,\ell+1}{;\zeta^{j,\ell+1}}) \right\|_{2}\notag \\
&\leq \frac{L_{R{,\zeta}}(\Delta^a_{\ell-1} + \varepsilon_2) + \varepsilon_1}{L}
\end{align}

Finally, for the Monte-Carlo error, note that, since $f$ is deterministic and $\zeta$ is independent from $Z$, using our hypothesis on $f$, we have 
\begin{multline*}
    {\E[f(s_\ell, a(s_\ell), Z^{j,\ell+1}(s_\ell){; \zeta^{j,\ell+1}})] = \E[\E[f(s_\ell, a(s_\ell), Z^{j,\ell+1}(s_\ell){; \zeta^{j,\ell+1}})|Z^{j,\ell +1}]]} \\= \E[f(s_\ell, a(s_\ell), Z^{j,\ell+1}(s_\ell){;\vec{0}})] = F(s_\ell, a(s_\ell))
\end{multline*}
from where it follows that the random vectors $(e^\ell_{\mathrm{mc}})_{\ell=1}^{L}$ are centered. 
{This is the only point where the dependency structure of $(Z^{j,\ell}, \zeta^{j,\ell})_{j,\ell}$ plays a role.}
\begin{enumerate}
    \item {If $(Z^{j,\ell}, \zeta^{j,\ell})_{j,\ell}$ are independent across $M$ and $L$, then $(e^\ell_{\mathrm{mc}})_{\ell=1}^{L}$ are also independent. {Furthermore, we know that they are subgaussian, with a variance proxy given by Lemma~\ref{lem:subgaussian_and_bounded_independent} $\|e^\ell_{\mathrm{mc}}\|^2_{vp} \leq \frac{c(1+\sigma_{R,\zeta}^2)}{ML^2}$, where $c>0$ is a constant depending on $B,R, B_\zeta, R_\zeta$. It follows that 
$\sum_{k=1}^{\ell} e^k_{\mathrm{mc}}$ is centered and subgaussian with variance proxy
$c\frac{1+\sigma_{R{,\zeta}}^2}{L M}$.}}

By concentration of subgaussian vectors, there
exists a constant $c > 0$ (with the required dependencies) such that for all $\delta > 0$ it holds that, with probability at least $1 - \delta$
\begin{equation}\label{eq:subgaussian_vector_concentration_for_mc_error}
    \left\| \sum_{k=1}^{\ell} e^k_{\mathrm{mc}} \right\|_2
  \leq c\, \frac{(1+\sigma_{R{,\zeta}})(\sqrt{D} + \sqrt{\log(1/\delta)})}{\sqrt{ML}}.
\end{equation}
As in \cite{chizat2025hidden}, {since the $(e_{mc}^k)_{k=1}^L$ are independent,} by the Lévy--Ottaviani inequality \cite[Lemma A.3]{chizat2025hidden}, the same bound holds for
$\max_{\ell \leq L} \|\sum_{k=1}^{\ell} e^k_{\mathrm{mc}}\|_2$ up to an absolute factor.
This argument allows to avoid an extra $\sqrt{\log L}$ factor that a union bound over
$\ell \in [1:L]$ would yield; alternatively, one could directly apply Azuma--Hoeffding's
lemma, as in \cite[Lemma A.1]{MeiMontanari2018mftwolayernetworks}).

\item 
    {If $(Z^{j,\ell}, \zeta^{j,\ell})_{j,\ell}$ are only independent across $L$, then $(e^\ell_{\mathrm{mc}})_{\ell=1}^{L}$ are still centered and independent, but the terms in the inner sum are not.}
    {
    Still, the vectors are subgaussian and their variance-proxy is simply bounded by the triangle inequality $\|e^{\ell}_{mc}\|_{vp} \leq \frac{c(1+\sigma_{R,\zeta})}{L}$, where we also used Lemma~\ref{lem:subgaussian_and_bounded_independent}. Since the error terms are centered and independent across $\ell\in[1:L]$, it follows that
    $\sum_{k=1}^{\ell} e^k_{\mathrm{mc}}$ is centered and subgaussian with variance proxy $\|\sum_{k=1}^{\ell} e^k_{\mathrm{mc}} \|_{vp} \leq c\frac{1+\sigma_{R{,\zeta}}}{\sqrt{L}}$. Applying the same concentration argument as in~\eqref{eq:subgaussian_vector_concentration_for_mc_error}, we obtain }
    \begin{equation}\label{eq:subgaussian_vector_concentration_for_mc_error_shared_in_M}
    \left\| \sum_{k=1}^{\ell} e^k_{\mathrm{mc}} \right\|_2
  \leq c\, \frac{(1+\sigma_{R{,\zeta}})(\sqrt{D} + \sqrt{\log(1/\delta)})}{\sqrt{L}}.
\end{equation}
    {Since $(e^\ell_{\mathrm{mc}})_{\ell=1}^{L}$ are still independent, $\sum_{k=1}^{\ell} e^k_{\mathrm{mc}}$ is still a martingale in $\ell$, and so the same bound holds for
$\max_{\ell \leq L} \|\sum_{k=1}^{\ell} e^k_{\mathrm{mc}}\|_2$ up to an absolute factor by the Lévy--Ottaviani inequality .}

    \item {If $(Z^{j,\ell}, \zeta^{j,\ell})_{j,\ell}$ are only independent across $M$, then $(e^\ell_{\mathrm{mc}})_{\ell=1}^{L}$ are still centered, but no longer independent.} 
    {Still, $e^\ell_{\mathrm{mc}}$ is a sum of independent and centered subgaussian terms, so it is subgaussian and has variance-proxy $\|e^{\ell}_{mc}\|_{vp}^2 \leq \frac{c(1+\sigma_{R,\zeta}^2)}{ML^2}$, as before. 
    Since the error terms now are not necessarily independent across $\ell\in[1:L]$, we bound the variance proxy of
    $\sum_{k=1}^{\ell} e^k_{\mathrm{mc}}$ by $\|\sum_{k=1}^{\ell} e^k_{\mathrm{mc}} \|_{vp} \leq c\frac{1+\sigma_{R{,\zeta}}}{\sqrt{M}}\frac{\ell}{L}$, using the triangle inequality. We apply the same concentration argument as in~\eqref{eq:subgaussian_vector_concentration_for_mc_error} to bound \(\big\| \sum_{k=1}^{\ell} e^k_{\mathrm{mc}} \big\|_2\). Since $\sum_{k=1}^{\ell} e^k_{\mathrm{mc}}$ is no longer a martingale, we perform a union bound to obtain} 
    \begin{equation}\label{eq:subgaussian_vector_concentration_for_mc_error_shared_in_L}
    \max_{\ell \leq L}\left\| \sum_{k=1}^{\ell} e^k_{\mathrm{mc}} \right\|_2
  \leq c\, \frac{{\sqrt{\log(L)}}(1+\sigma_{R{,\zeta}})(\sqrt{D} + \sqrt{\log(1/\delta)})}{\sqrt{M}}.
\end{equation}
\end{enumerate}
{All in all, introducing $\beta$ as in the statement of the lemma, we have:
\begin{equation}\label{eq:subgaussian_vector_concentration_for_mc_error_all}
    \max_{\ell \leq L}\left\| \sum_{k=1}^{\ell} e^k_{\mathrm{mc}} \right\|_2
  \leq c\, {\beta}\,\frac{(1+\sigma_{R{,\zeta}})(\sqrt{D} + \sqrt{\log(1/\delta)})}{\sqrt{ML}}.
\end{equation}}

Plugging all these error estimates into \eqref{eq:error_control_stoch_approx}, we obtain that with probability at
least $1 - \delta$, for $\ell \in [1:L]$, provided $\Delta^a_k \leq R$ for
$k \in [1:\ell-1]$, it holds
\[
  \Delta^a_\ell
  \leq \varepsilon_0 + \frac{L_{R{,\zeta}}}{L} \sum_{k=0}^{\ell-1} \Delta^a_k
  + \varepsilon_1+ L_{R{,\zeta}} \varepsilon_2
  + \frac{c}{L}
  + c\beta\, \frac{{(1+\sigma_{R{,\zeta}})}(\sqrt{D} + \sqrt{\log(1/\delta)})}{\sqrt{LM}}
\]
where $c$ only depends on $L_{R{,\zeta}}$, $L_{{x}}$, $\Gamma$, $\|a(0)\|_2$, $B$, $R_\zeta$, $B_\zeta$. 
By
recursion (the discrete Gr\"onwall lemma),
we obtain,
\[\Delta^a_\ell
  \leq \exp\Big(cL_{R{,\zeta}} \Big) \Big(\varepsilon_0 + \varepsilon_1+ L_{R,\zeta} \varepsilon_2
  + \frac{c}{L}
  + c\beta\, \frac{{(1+\sigma_{R{,\zeta}})}(\sqrt{D} + \sqrt{\log(1/\delta)})}{\sqrt{LM}}\Big)\]
The condition on $\Delta^a_k$ is satisfied
by requiring the upper bound to be smaller than $R$, which is precisely what the control
by $c_2$ in the statement achieves. This is precisely what we wanted to prove.
\end{proof}

\subsection{Proof of Theorem~\ref{thm:cvg_to_limit}}\label{app:proof_main_thm}
Consider the ResNet's dynamics $(\hat{Z}_k^{j,\ell})$ defined in~\eqref{eq:parameter_update_resnet_dropout} and consider the $M \times L$ (possibly correlated) copies of the limit dynamics $(Z_k^{j,\ell})_{k \geq 0}$ defined in~\eqref{eq:limit_update_rule_ram};  coupled via $Z_0^{j,\ell}(s) = \hat{Z}_0^{j,\ell}, \forall s \in [0,1]$; and via the masks $\zeta^{j,\ell}_k$ used for the $(j,\ell)$-th copy of the limit system. Note that, {if the masks $(\zeta^{j,\ell}_k)_k$ are independent across $M$, then so are the copies $(Z^{j,\ell}_k)_k$; and similarly if they were independent across $L$.} 
Consider the distance between the two dynamics in parameter space, forward pass and backward pass respectively defined, with $s_\ell = \ell/L$ for $\ell \in [0:L]$, as
\[
    \Delta_k^Z := \max_{\substack{j \in [1:M] \\ \ell \in [1:L]}} \|\hat{Z}_k^{j,\ell} - Z_k^{j,\ell}(s_{\ell-1})\|, \quad
    \Delta_k^h := \max_{\substack{i \in [1:n] \\ \ell \in [0:L]}} \|\hat{h}_{k,i}^\ell - h_{k,i}(s_\ell)\|, \quad
    \Delta_k^b := \max_{\substack{i \in [1:n] \\ \ell \in [0:L]}} \|\hat{b}_{k,i}^\ell - b_{k,i}(s_\ell)\|.
\]

\paragraph{Step 1: Set-up.}
In this proof, we denote by $c$ a positive real number
that may depend on $B$, $k\tau$ and $D$, and that may change from line to line. First,
by Lemma~\ref{lem:propagation_of_regularity}, there exists $R > 0$ that only depends on $B$ such that
\[
  \max_{i \leq n,\, k' \leq k,\, s \in [0,1]}
  \bigl(\|h_{k',i}(s)\|_2 \vee \|b_{k',i}(s)\|_2\bigr) \leq R.
\]
Let $\tilde{k} \leq k$ be the largest index such that
$\max_{i \leq n,\, k' \leq k,\, \ell \leq L}
(\|\hat{h}^\ell_{k,i}\|_2 \vee \|\hat{b}^\ell_{k,i}(s)\|_2) \leq 2R$ holds. We now
work with time horizon $\tilde{k}$, and will show at the end of the proof that $\tilde{k}$
can be made equal to $k$ with high probability provided $c_2$ is small enough. Recall that
\begin{align*}
\hat{Z}^{j,\ell}_{k+1}
&= \hat{Z}^{j,\ell}_k
  - \frac{\tau}{n} \sum_{i=1}^{n}
    \tilde{\mathfrak{u}}(\hat{Z}^{j,\ell}_k, \hat{h}_{\theta_k}^{\zeta}(\ell-1, x_i), \hat{b}_{\theta_k}^{\zeta}(\ell, x_i, \hat{w}_{i,k})) (1+\zeta^{\ell,j}_k), \\
Z_{k+1}^{j,\ell}(s)
&= Z_k^{j,\ell}(s)
  - \frac{\tau}{n} \sum_{i=1}^{n}
    \tilde{\mathfrak{u}}(Z_k^{j,\ell}(s), h_{\mu_k}(s, x_i), b_{\mu_k}(s, x_i, w_{i,k}))(1+\zeta_k^{j,\ell})
\end{align*}
for the same map $\tilde{\mathfrak{u}}(z, h, b) = D_1\phi(z,h)^\top \diag(b)$. We also wrote $\hat{w}_{i,k} = w_{x_i,y_i}^{\zeta_k}$ and $w_{i,k} = \bar{w}_{x_i,y_i}$ for simplicity. Under Assumption \ref{ass:regularity_assumptions}, the map $\tilde{\mathfrak{u}}$
restricted to $b$ in a ball of radius $2R$ is $L_R$-Lipschitz {from $\Vert\cdot\Vert_\infty$ to $\Vert\cdot\Vert_2$} (as seen in Remark~\ref{rem:rewritten_update_map}). Therefore, for all
$k \leq \tilde{k}$,
\begin{equation}
  \Delta^Z_{k+1} \leq \Delta^Z_k + \tau L_R(\Delta^Z_k + \Delta^h_k + \Delta^b_k)\max_{j,\ell}\|1 + \zeta_k^{\ell,j}\|{_\infty}.  \notag
\end{equation}
Let us now fix $k \leq \tilde{k}$ and control these various terms with high probability. We note that a.s. \(\|1 + \zeta_k^{\ell,j}\|{_{\infty}}\leq (1+B_\zeta)\), so that:
\begin{equation}\label{eq:proof_thm_basic_bound}
  \Delta^Z_{k+1} \leq \Delta^Z_k + \tau(1+B_\zeta) L_R(\Delta^Z_k + \Delta^h_k + \Delta^b_k).  
\end{equation}

\paragraph{Step 2: Control on $\Delta^h_k$.} Let us verify that we can apply Lemma~\ref{lem:stoch_approx} with $f = f^h_k : (s, a, z{,\xi}) \mapsto \phi(z,a){\odot(1+\xi)}$ and $Z = Z_k$. Clearly
$f^h_k$ is Lipschitz and by Lemma~\ref{lem:propagation_of_regularity}, $s \mapsto Z_k(s)$ also. By Lemma~\ref{lem:propagation_of_subgaussianity} and
composing with a Lipschitz function, $f^h_k(s, a, Z_k(s){;\xi})$ is subgaussian with variance
proxy $c{(1+\sigma_0^2)}$ (where $c$ depends only on $B$, $D$, {$B_\zeta$} and $k\tau$). Therefore the
proposition applies (with $\varepsilon_0 = \varepsilon_1 = 0$ and
$\varepsilon_2 = \Delta^Z_k$). By a union bound over $i \in [1:n]$, there exists
$c_{1,k}$ such that with probability at least $1 - \delta$, it holds
\[
  \Delta^h_k
  \leq c_{1,k} \left(
    \Delta^Z_k + \frac{1}{L}
    + {\beta} \,\frac{{(1+\sigma_0)}(1 + \sqrt{\log(n/\delta)})}{\sqrt{ML}}
  \right).
\]

\paragraph{Step 3: Control on $\Delta^b_k$.} Let us verify that we can apply
Lemma~\ref{lem:stoch_approx} with $f^b_k : (s, a, z{,\xi}) \mapsto D_2\phi(z,h_k(s, x_i))^\top (a{\odot(1+\xi)})$,
$Z = Z_k$ and $\hat{f}^b_k : (s_\ell, a, z) \mapsto D_2\phi(z,\hat{h}^\ell_k(x_i))^\top (a{\odot(1+\xi)})$.
Although $f^b_k$ is not globally Lipschitz, it is Lipschitz when its argument $x$ is
restricted to a ball by Lemma~\ref{lem:propagation_of_regularity} (the assumptions of Lemma~\ref{lem:stoch_approx} are designed
precisely to cover this case). Also, by the subgaussian tail estimates of Lemma~\ref{lem:propagation_of_subgaussianity},
$f^b_k(s, a, Z_k(s){;\xi})$ is subgaussian with variance proxy $c{(1+\sigma_0^2)}$ (where $c$
depends only on $B$, $D$, ${B_\zeta}$ and $k\tau$). Therefore the proposition applies (with
$\varepsilon_0 = \|w_{i,k} - \hat{w}_{i,k}\|_2 \leq c\Delta^h_k$,
$\varepsilon_1 \leq c\Delta^h_k$ and $\varepsilon_2 \leq \Delta^Z_k$). By a union bound
over $i \in [1:n]$, there exists $c_{2,k}$ such that with probability at least $1 - \delta$,
it holds
\[
  \Delta^b_k
  \leq c_{2,k} \left(
    \Delta^Z_k + \Delta^h_k + \frac{1}{L}
    + {\beta}\,\frac{{(1+\sigma_0)}(1 + \sqrt{\log(n/\delta)})}{\sqrt{ML}}
  \right).
\]

\paragraph{Step 4: Conclusion.} Take a union bound over the at most $2 \times k$ events
where all the previous bounds hold for $k' \leq \tilde{k} \vee (k-1)$. Plugging into
\eqref{eq:proof_thm_basic_bound}, there exists $c$ (depending only on $B$, $D$, ${B_\zeta}$ and $k\tau$) such that with probability at least $1 - \delta$, for
$0 \leq k' \leq \tilde{k} \vee (k-1)$, it holds
\[
  \Delta^Z_{k'+1}
  \leq \Delta^Z_{k'}
  + c \left(
    \Delta^Z_{k'} + \frac{1}{L}
    + {\beta}\,\frac{{(1+\sigma_0)}(1 + \sqrt{\log(kn/\delta)})}{\sqrt{ML}}
  \right),
\]
where, since $\sigma_0 \leq B$, we can absorb $\sigma_0$ into the constant.
Since $\Delta^Z_0 = 0$, the conclusion for $(\Delta^Z_k)$ follows by the discrete Gr\"onwall
lemma, and for $(\Delta^h_k)$ and $(\Delta^b_k)$ by the bounds in Step 2 and Step 3
respectively. Finally, by taking $c_2$ small enough in the statement of the theorem, we
can ensure $\tilde{k} \geq k$ under the same event.
This concludes the proof.

\paragraph{Case of GD-RaM}
{The proof of Theorem~\ref{thm:cvg_to_limit} works identically if we compare the ResNet's dynamics $(\tilde{Z}_k^{j,\ell})$ with RaM defined in~\eqref{eq:parameter_update_resnet_ram} to the $M \times L$ (possibly correlated) copies of the limit dynamics $(Z_k^{j,\ell})_{k \geq 0}$ defined in~\eqref{eq:limit_update_rule_ram}} and coupled as in the proof of Theorem~\ref{thm:cvg_to_limit}. Indeed, Lemma~\ref{lem:stoch_approx} applies with $\zeta^{j,\ell} \equiv 0$, and the rest of the estimates remain unchanged.

\subsection{Proof of Theorem~\ref{thm:lazy_thm}}\label{app:proof_lazy_thm}
In the setting from Section~\ref{sec:equivalence_in_lazy_ode}, we are in the so-called \emph{lazy ODE regime} for ResNet training where the limit is described by the \emph{Neural Tangent ODE}: let $\xi:[0,1]\to\R^p$ be a random process and $Z_0\in \R^p$ a random variable representing the initialization. Informally, $\xi$ represents the rescaled deviation from initialization: $\xi = \alpha(Z-Z_0)$, where $Z$ is the stochastic process driving the Mean ODE. 
The forward pass $\underline{h}_\xi(s, x) \in \mathbb{R}^D$ (with an implicit dependency in the law of $Z_0$) is the solution to the (forward Neural) \emph{tangent ODE}, for $s\in[0,1]$,
\begin{equation}\label{eq:forward_tangent_ode}
    \underline{h}_\xi(0, x) = x, \quad \partial_s \underline{h}_\xi(s, x) = \mathbf{E}[D_2 \phi(Z_0,\underline{h}_\xi(s, x)) \xi(s)].
\end{equation}
which is linear in the parameter $\xi$, however the output $\underline{h}_\xi(1, x)$ is still a nonlinear function of $x$ and $\xi$. Note that, once again, the mask $(\zeta^{j,\ell, d})_{j,\ell, d}$ disappears in the forward pass in the limit.
The corresponding backward tangent ODE is the solution to
\begin{equation}\label{eq:backward_tangent_ode}
\underline{b}_\xi(1, x, w) = w,\quad \partial_s \underline{b}_\xi(s, x, w) = -\mathbf{E}\left[ D_{2,1} \phi(Z_0,\underline{h}_\xi(s, x))^{*2} [\xi(s), \underline{b}_\xi(s, x, w)] \right],   
\end{equation}
where $D_{2,1} \phi(x, z)^{*1}$ is the partial adjoint in the second variable of the mixed second derivative of $\phi$, which is interpreted as a linear operator $\mathbb{R}^p\times\mathbb{R}^D \to \mathbb{R}^D$. The update equations that drive the training dynamics correspond to GD in the $L^2$ geometry, initialized at 0, for the random processes $(\xi_k)_{k \geq 0}$:
\begin{equation}\label{eq:update_lazy_regime}
    Z_0(s) \sim \mu_0, \quad \xi_0(s) = 0, \quad \xi_{k+1}(s) = \xi_k(s) - \tau \E_{\pi}[\tilde{\mathfrak{u}}(Z_0(s), \underline{h}_{\xi_k}(s,x), \underline{b}_{\xi_k}(s,x,\underline{w}_{x,y}))](1+\zeta_k),
\end{equation}
where $(\zeta_k)_k \subseteq\R^D$ are independently sampled 
{bounded} masks, as in~\eqref{eq:limit_update_rule_ram}.

As in \cite[Theorem 2]{chizat2025hidden}, one can prove a quantitative bound for convergence to this limit dynamics.

Let $(\hat{Z}_k^{j,\ell})$ be the iterates of the scaled dropout ResNet dynamics~\eqref{eq:masked_resnet} and let $((Z_0^{j,\ell}, \xi_k^{j,\ell})_k)_{j,\ell}$ denote $L\times M$ copies of the limit dynamics~\eqref{eq:update_lazy_regime}, coupled to $(\hat{Z}_k^{j,\ell})$ via $Z_0^{j,\ell}(s) = \hat{Z}_0^{j,\ell}, \forall s \in [0,1]$; and via the mask $\zeta^{j,\ell}_k$ that is used for updating the $(j,\ell)$-th copy of the limit system.
Define the distances between the two dynamics as
\[
    {\Delta}_k^\xi := \max_{\substack{j \in [1:M] \\ \ell \in [1:L]}} \|\alpha(\hat{Z}_k^{j,\ell}-\hat{Z}_0^{j,\ell}) - \xi_k^{j,\ell}(s_{\ell-1})\|, \quad
    {\Delta}_k^{\underline{h}} := \max_{\substack{i \in [1:n] \\ \ell \in [0:L]}} \|\hat{h}_{k,i}^\ell - \underline{h}_{k,i}(s_\ell)\|.
\]
and similarly for ${\Delta}_k^{\underline{b}}$; we have the following result.

\begin{theorem}\label{thm:cvg_to_lazy_limit}
Let Assumption~\ref{ass:regularity_assumptions} hold with $B>0$, let $\alpha \geq 1$ and let $\mu_0\in\P(\R^p)$ be a subgaussian distribution with variance proxy $\sigma_0^2 \leq B$.
Further assume that $\phi$ is twice differentiable with a $B$-Lipschitz cross differential $D_{2,1}\phi$, and that for all $h\in \R^D$ $\E_{\mu_0}[\phi(h, Z_0)] = \E_{\mu_0}[D_1\phi(h, Z_0)] = 0$. 

If the dropout masks $(\zeta_k)_k$ are {bounded by $B_\zeta$}
then, for all $k\geq 1$, there exist $c_1, c_2 > 0$ that only depend on $B$, $D$, $k\tau$ and
$B_\zeta$, such that for all $\delta\in(0,1)$ with probability at least $1 - \delta$, it holds:
\[
  \max_{k'\leq k} (\Delta_{k'}^\xi, \Delta_{k'}^{\underline{h}}, \Delta_{k'}^{\underline{b}})
  \leq c_1 \left(\frac{1}{\alpha} + \frac{1}{L}
    +  {\beta}\frac{\alpha(1 + \sqrt{\log(kn/\delta)})}{\sqrt{LM}}
  \right)
\]
provided that the right-hand side is smaller than $c_2$. {The multiplier $\beta$ is defined as in Theorem~\ref{thm:main}}.
\end{theorem}

We follow the lines of the proof of~\cite[Theorem 2]{chizat2025hidden} very closely. Since the proof is very similar to that of Theorem~\ref{thm:cvg_to_limit}, we provide less details and focus on the differences.

The proof is very similar to that of Theorem~\ref{thm:cvg_to_limit}, with an additional error term coming from the linearization of $\phi$ in its first argument. Throughout the proof, $c$ denotes a positive real number that depends only on $B$, $k\tau$, $D$ and $B_\zeta$, and that may change from line to line.

\paragraph{Step 1: Set-up} The functions appearing in the Mean ODE of the limit model are now:
$f_k^h : (s,x,z{;\xi}) = (D_1\phi(z_0,x)z_1){\odot (1+\xi)}$ where $z = (z_0,z_1) \in \R^{2p}$ (for the forward pass); and for the backward pass: $f_k^b : (s,x,z{;\xi}) = D_{2,1}\phi(z_0,\underline{h}_{k,i}(s))^{*2}[z_1,x{\odot(1+\xi)}] \in \R^{D}$ where $z = (z_0,z_1) \in \R^{2p}$. 
Note again that ${f_k^h(\cdot; \vec{0})}$ and ${f_k^b(\cdot;\vec{0})}$ coincide with the functions from the limit model and those of the vanilla case.

Under the assumptions of Theorem~\ref{thm:cvg_to_lazy_limit}, one can prove using similar arguments as in Sections~\ref{app:properties_mean_ode} that $\underline{h}_{k,i}$ and $\underline{b}_{k,i}$ have the same regularity properties as $h_{k,i}$ and $b_{k,i}$ respectively. It suffices to augment Lemma~\ref{lem:propagation_of_regularity} with the property that $\sup_{s \in [0,1]} \|\xi_k(s)\|$ is bounded by a constant depending on $B$, {$B_\zeta$} and $k\tau$ only; and similarly Lemma~\ref{lem:propagation_of_subgaussianity} bounding the subgaussian norm of $Z_k(s) = Z_0 + \alpha^{-1}\xi_k(s)$ by $c{(1+\sigma_0)}$.

Consider, as in Appendix~\ref{app:proof_main_thm}, a (random) index $\tilde{k}$ such that the norms of $\underline{h}_{k',i}, \underline{b}_{k',i}, \hat{h}_{k',i}$ and $\hat{b}_{k',i}$ are controlled by $2R$ for all $k' \le \tilde{k} \vee k$ and $i \in [1:n]$. Recalling
\[
    \Delta_k^\xi := \max_{\substack{j \in [1:M] \\ \ell \in [1:L]}} \|\alpha(\hat{Z}_k^{j,\ell} - \hat{Z}_0^{j,\ell}) - \xi_k^{j,\ell}(s_{\ell-1})\|
\]
and from the expression of the updates, it holds (a factor $\alpha$ gets canceled out):
\begin{equation}\label{eq:update_bound_lazy}
    \Delta_{k+1}^\xi \le \Delta_k^\xi + c\tau(\Delta_k^{\underline{h}} + \Delta_k^{\underline{b}}) \max_{j,\ell}\|\vec{1}+\zeta^{j,\ell}\|_{{\infty}}.
\end{equation}

\paragraph{Step 2: Control on $\Delta_k^{\underline{h}}$.} We would like to apply Lemma~\ref{lem:stoch_approx}, but it is not possible since the regularity estimates of the map $\alpha\phi$ diverge. We perform, as in~\cite{chizat2025hidden} the following alternative error decomposition:
    \[
        \begin{aligned}
            \underline{h}_{k,i}(s_{\ell+1}) - \hat{h}_{k,i}^{\ell+1} &= \underline{h}_{k,i}(s_\ell) - \hat{h}_{k,i}^\ell \\
            &+ \int_{s_\ell}^{s_{\ell+1}} \mathbb{E}[D_1\phi(Z_0,\underline{h}_{k,i}(s))\xi_k(s)] \mathrm{d}s - \frac{1}{L} \mathbb{E}[D_1\phi(Z_0,\underline{h}_{k,i}(s_\ell))\xi_k(s_\ell)] \\
            &+ \frac{1}{L} \mathbb{E}[D_1\phi(Z_0, \underline{h}_{k,i}(s_\ell))\xi_k(s_\ell)] - \frac{1}{L} \mathbb{E}[D_1\phi(Z_0,\hat{h}_{k,i}^\ell)\xi_k(s_\ell)] \\
            &+ \frac{1}{L} \mathbb{E}[D_1\phi(Z_0,\hat{h}_{k,i}^\ell)\xi_k(s_\ell)] - \frac{\alpha}{L} \mathbb{E}[\phi(Z_0 + \alpha^{-1}\xi_k(s_\ell), \hat{h}_{k,i}^\ell)] \\
            &+ \frac{\alpha}{L} \mathbb{E}[\phi(Z_0 + \alpha^{-1}\xi_k(s_\ell), \hat{h}_{k,i}^\ell)] - \frac{\alpha}{ML} \sum_{j=1}^M \phi(Z_0^{j,\ell+1} + \alpha^{-1}\xi_k^{j,\ell+1}, \hat{h}_{k,i}^\ell){\odot}(\vec{1}+\zeta_k^{j,\ell}) \\
            &+ \frac{\alpha}{ML} \sum_{j=1}^M  \Big(\phi(Z_0^{j,\ell+1} + \alpha^{-1}\xi_k^{j,\ell+1}, \hat{h}_{k,i}^\ell) - \phi(Z_0^{j,\ell+1} + \alpha^{-1}\hat{\xi}_k^{j,\ell+1}, \hat{h}_{k,i}^\ell)\Big){\odot}(\vec{1}+\zeta_k^{j,\ell})
        \end{aligned}
    \]
    The terms from line to line correspond respectively to:
    \begin{itemize}
        \item the Euler discretization error bounded by $c/L^2$,
        \item a term bounded by $(c/L)\|\underline{h}_{k,i}(s_\ell) - h_{k,i}^\ell\|_2$,
        \item the linearization error bounded by $c/(\alpha L)$ (this is the ``new'' term as compared to Lemma~\ref{lem:stoch_approx}),
        \item the Monte-Carlo sampling error multiplied by $\alpha$. Unlike~\cite{chizat2025hidden}, we now have the mask $\zeta_k$.
        \item  term bounded by $c\Delta_k^\xi / L$ (using that $\|\vec{1}+\zeta_k^{j,\ell}\|_{\infty} \leq B_\zeta$ and absorbing this into $c$).
    \end{itemize}
     As is also the case in \emph{vanilla} networks (see~\cite{chizat2025hidden}), the Monte-Carlo term must be handled differently since $\hat{h}_{k,i}^\ell$ is not independent from $(Z_0^{j,\ell+1}, \xi_k^{j,\ell+1})$. To handle this, a \emph{uniform} concentration bound over $\hat{h}_k^\ell$ in the ball of radius $2R$ in $\R^D$ must be used.
     Indeed, with probability at least $1-\delta$,
     \begin{multline*}
         \sup_{h\in B(0, 2R)}\Big\|\sum_{\ell=1}^{\ell'}\frac{\alpha}{L} \mathbb{E}[\phi(Z_0 + \alpha^{-1}\xi_k(s_\ell), h)] - \frac{\alpha}{ML} \sum_{j=1}^M (\phi(Z_0^{j,\ell+1} + \alpha^{-1}\xi_k^{j,\ell+1}, h)){\odot}(\vec{1}+\zeta_k^{j,\ell})\Big\|_2\\
         \leq c\, \frac{\alpha{(1+\sigma_R)}(1 + \sqrt{\log(1/\delta)})}{\sqrt{ML}}
     \end{multline*}
     where we used the fact that $(Z_0^{j,\ell}, \xi^{j,\ell}, \zeta^{j,\ell})$ are identically distributed and that $\E[\diag(\phi(Z_0^{j,\ell+1} + \alpha^{-1}\xi_k^{j,\ell+1}, h))(\vec{1}+\zeta_k^{j,\ell})] = \mathbb{E}[\phi(Z_0 + \alpha^{-1}\xi_k(s_\ell), h)]$ by the independence of $\zeta$ from everything else ($\xi_k$ depends only on $\zeta$ from past iterations). The bound now depends on other terms involving $D$ and $R$ (coming from an $\varepsilon$-net argument to bound the probability of the $\sup_{h\in B(0,2R)}$) that are absorbed into $c$. {The case of shared masks is treated as in the proof of Lemma~\ref{lem:stoch_approx}.}

    Proceeding as in the proof of Lemma~\ref{lem:stoch_approx}, we get that with probability at least $1-\delta$,
    \[
        \Delta_k^{\underline{h}} \le c_{1,k} \left( \Delta_k^\xi + \frac{1}{\alpha} + \frac{1}{L} + {\beta}\frac{\alpha{(1+\sigma_0)}(1+\sqrt{\log(n/\delta)})}{\sqrt{LM}} \right).
    \]
    To bound $\Delta_k^{\underline{b}}$, one proceeds analogously, by performing a similar error decomposition as above, and following \emph{Step 3} from the proof of Section~\ref{app:proof_main_thm}. We obtain that with probability at least $1-\delta$,
    \[
        \Delta_k^{\underline{b}} \le c_{2,k} \left( \Delta_k^\xi + \Delta_k^{\underline{h}} + \frac{1}{\alpha} + \frac{1}{L} + {\beta}\frac{\alpha{(1+\sigma_0)}(1+\sqrt{\log(n/\delta)})}{\sqrt{LM}} \right).
    \]

    \paragraph{Conclusion.} As in Section~\ref{app:proof_main_thm}, we take a union bound over the at most $2k$ events where the previous bounds hold up to $\tilde{k} \vee (k-1)$, we plug these estimates into~\eqref{eq:update_bound_lazy} and conclude as in the \emph{Conclusion} of Section~\ref{app:proof_main_thm} (ensuring that $\tilde{k} \ge k$ by taking $c_2$ small enough).

    \paragraph{Proof of Theorem~\ref{thm:lazy_thm}}
    As in the proof of Theorem~\ref{thm:main}, the proof of Theorem~\ref{thm:lazy_thm} follows from Theorem~\ref{thm:cvg_to_lazy_limit} by a triangle inequality.

\section{Additional Numerical Experiments}\label{app:additional_numerics}
We consider the same MNIST setting described in Section~\ref{sec:numerics}.

In Figure~\ref{fig:params_mnist} we show certain projections of the parameters at the last iteration of training. In order to visualize regularity, only the last particle is tracked, it is initialized with the same value for all depths and a depth-shared mask. 
We can see that parameters from the Dropout/RaM dynamics from the independent, width-shared and depth-shared variants remain close, while the GD counterpart behaves distinctly.
\begin{figure}
  \centering
  \includegraphics[width=0.85\textwidth]{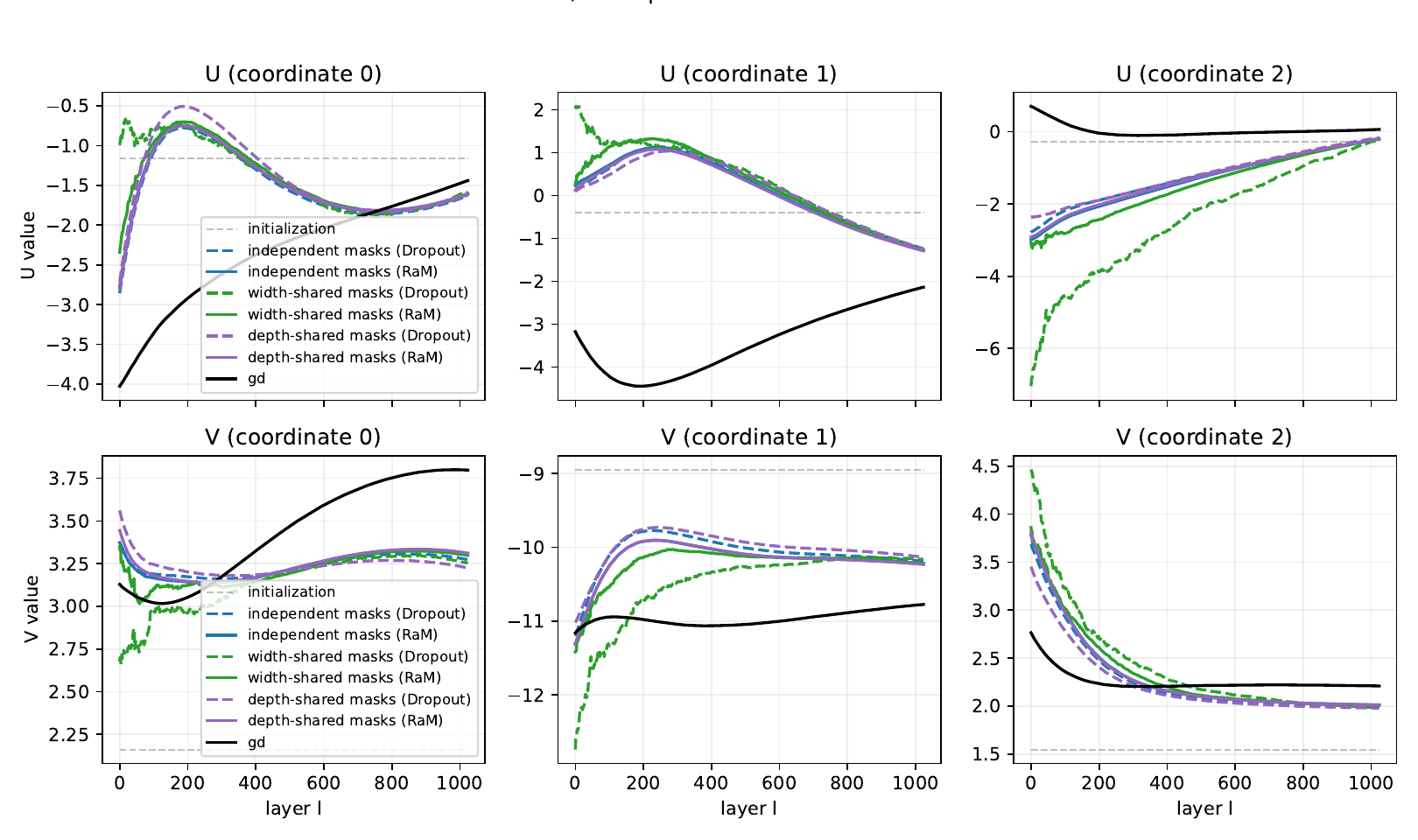}
  \caption{Comparison of certain projections of the parameters at the last training iteration ($k=200$), for SGD-Dropout, SGD-RaM and vanilla SGD, for different masking variants. The employed ResNet has $D=10$, $L=M=1024$. For all variants, parameters remain essentially close, and distinct from vanilla SGD.}
  \label{fig:params_mnist}
\end{figure}

 Figure~\ref{fig:param_error_mnist} shows the RMS distance in terms of $M$ and $L$, as well as the ``effective model width'' $ML$, for the network parameters (rather than the forward passes in Figure~\ref{fig:forward_error_mnist}). The observations are similar than the ones from Figure~\ref{fig:forward_error_mnist}: the $O(1/\sqrt{ML})$ rate is quite clear for the fully independent masks, while the mask sharing schemes converge at rates $O(1/\sqrt{L})$ for the \emph{width-shared} variant, and $O(1/\sqrt{M})$ for the \emph{depth-shared} one. Figure~\ref{fig:details_forward_error_mnist} provides a more detailed view of the results displayed in Figure~\ref{fig:forward_error_mnist}, further supporting our claims. Finally, Figure~\ref{fig:details_forward_error_mnist_k200} shows the same quantities but computed after a longer training horizon ($k=200$). The convergence bounds from Theorem~\ref{thm:main} degrade with $k$, so it is not a surprise that the slopes do not match as cleanly as with $k=50$. Yet, the qualitative behavior remain consistent with our analysis even at this larger training horizon.

\begin{figure}
        \centering
        \begin{subfigure}{0.635\linewidth}
        \centering
        \includegraphics[width=\linewidth]{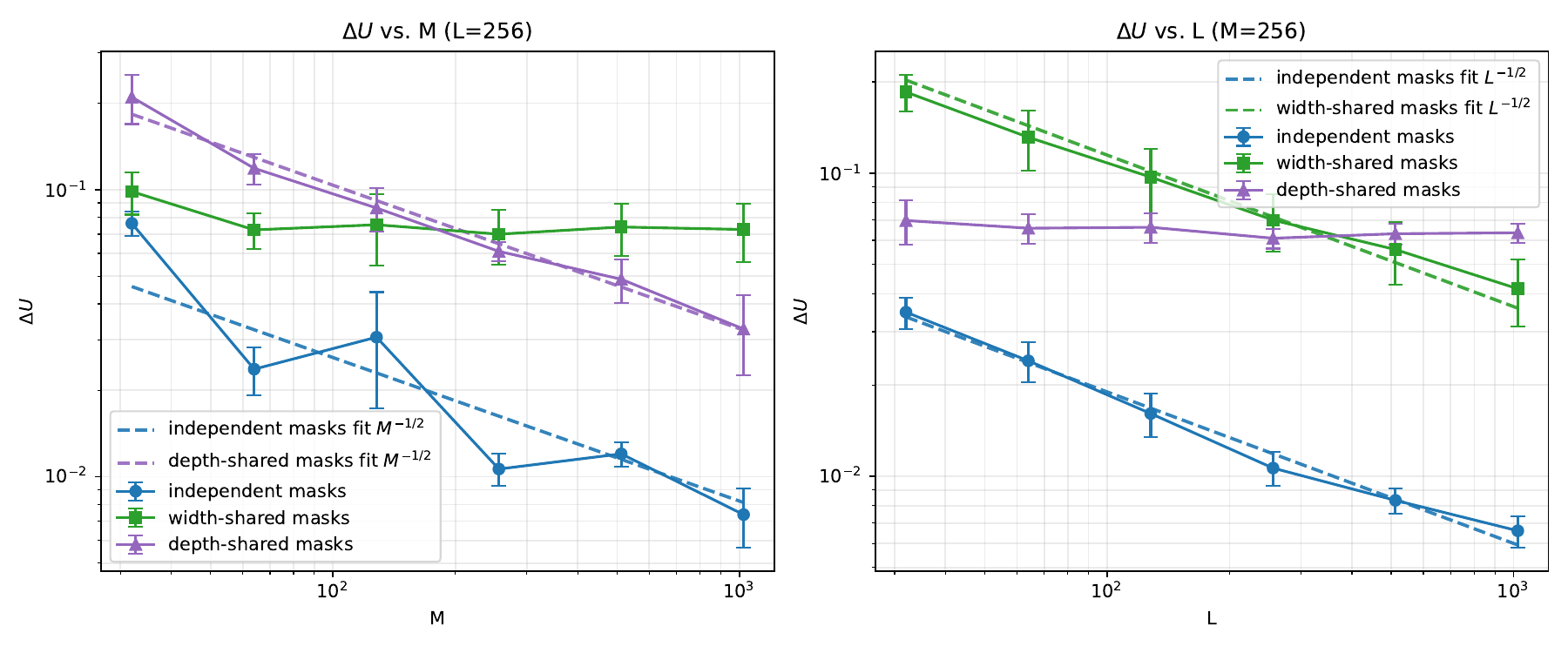}
        \caption{$\Delta^U_k$ against both $M$ and $L$ separately}
        \end{subfigure}%
        \begin{subfigure}{0.365\linewidth}
        \centering
        \includegraphics[width=\linewidth]{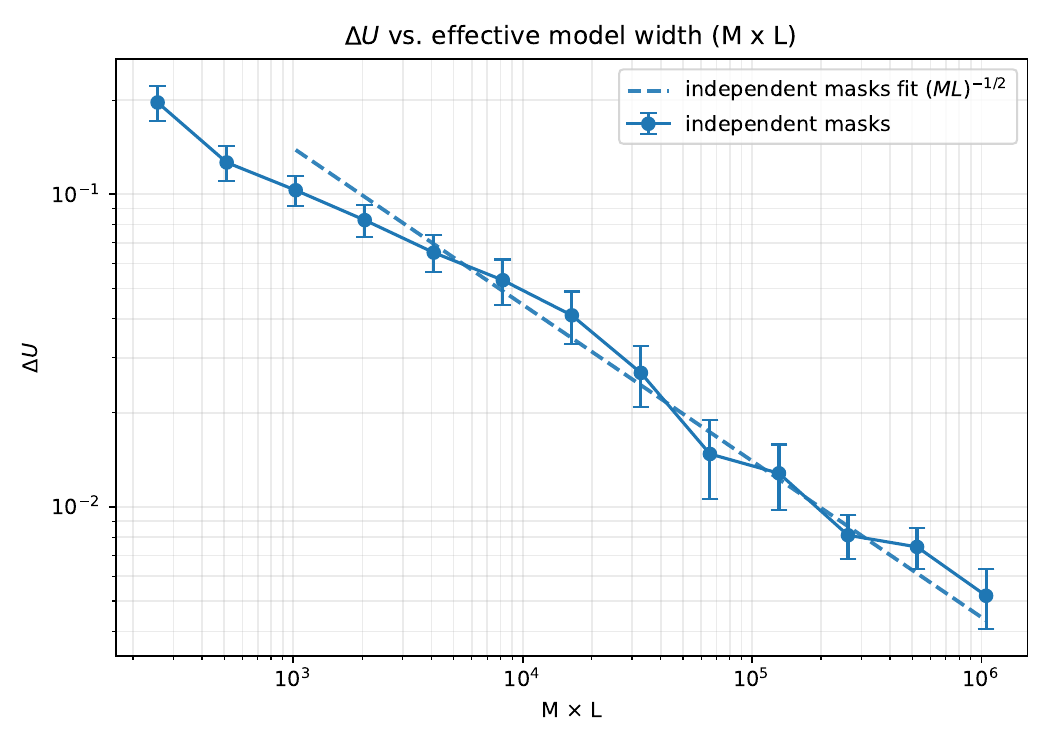}
        \caption{$\Delta^U_k$ vs $ML$}
        \end{subfigure}\\
        \begin{subfigure}{0.635\linewidth}
        \centering
        \includegraphics[width=\linewidth]{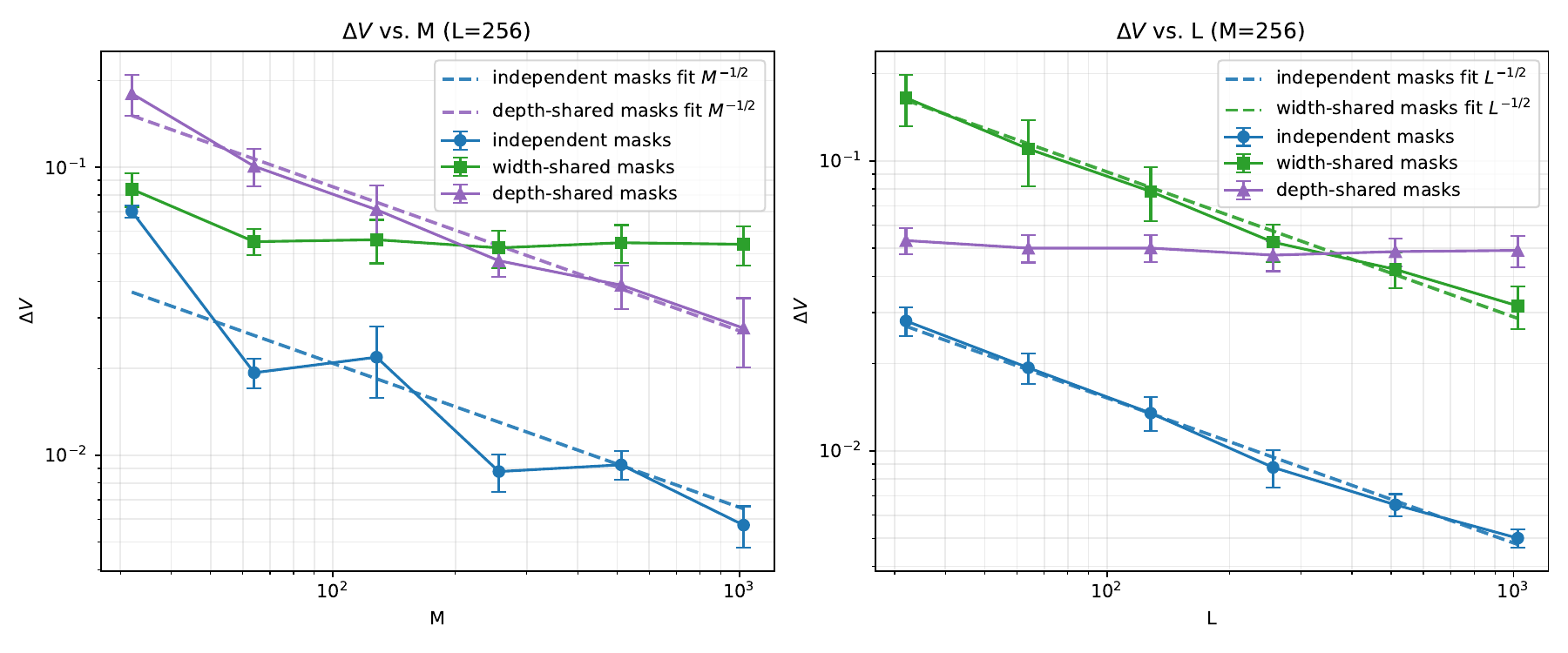}
        \caption{$\Delta^V_k$ against both $M$ and $L$ separately}
        \end{subfigure}%
        \begin{subfigure}{0.365\linewidth}
        \centering
        \includegraphics[width=\linewidth]{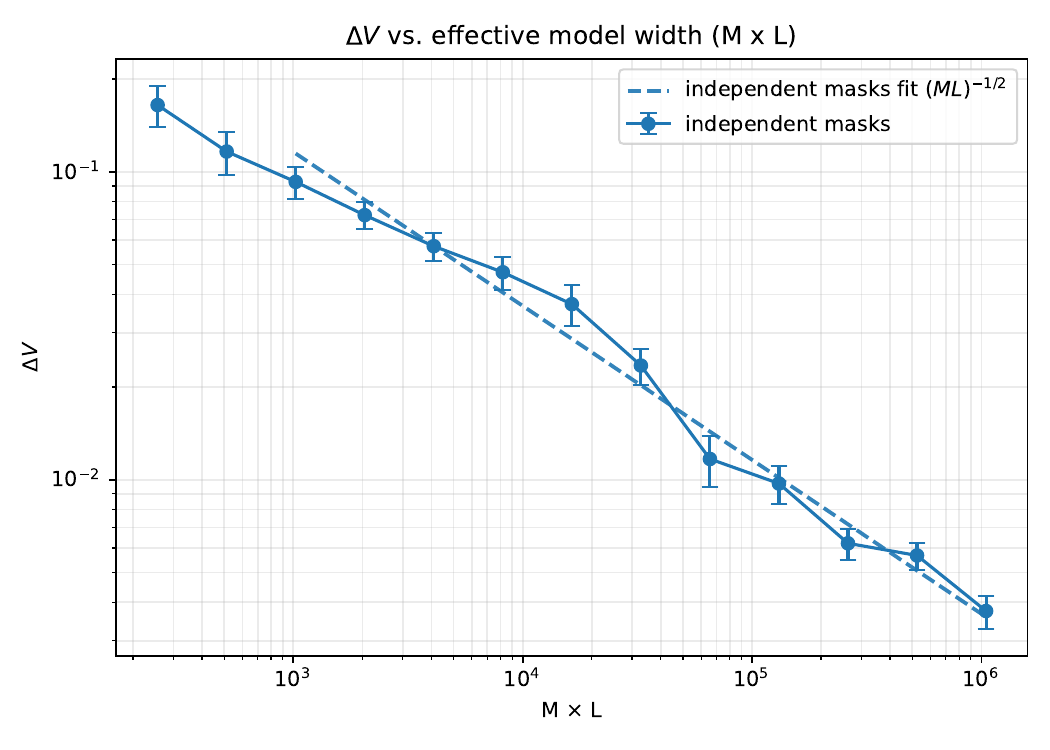}
        \caption{$\Delta^V_k$ vs $ML$}
        \end{subfigure}

        \caption{RMS distance between RaM and Dropout for the different masking strategies, as functions of $M$, $L$ and the ``effective model width'' $ML$. We consider the RMS error in parameter space $\Delta^U_k$, $\Delta^V_k$ at the training iteration $k=50$. }
  \label{fig:param_error_mnist}
\end{figure}

\begin{figure}
        \centering
        \begin{subfigure}{\linewidth}
        \centering
        \includegraphics[width=0.8\linewidth]{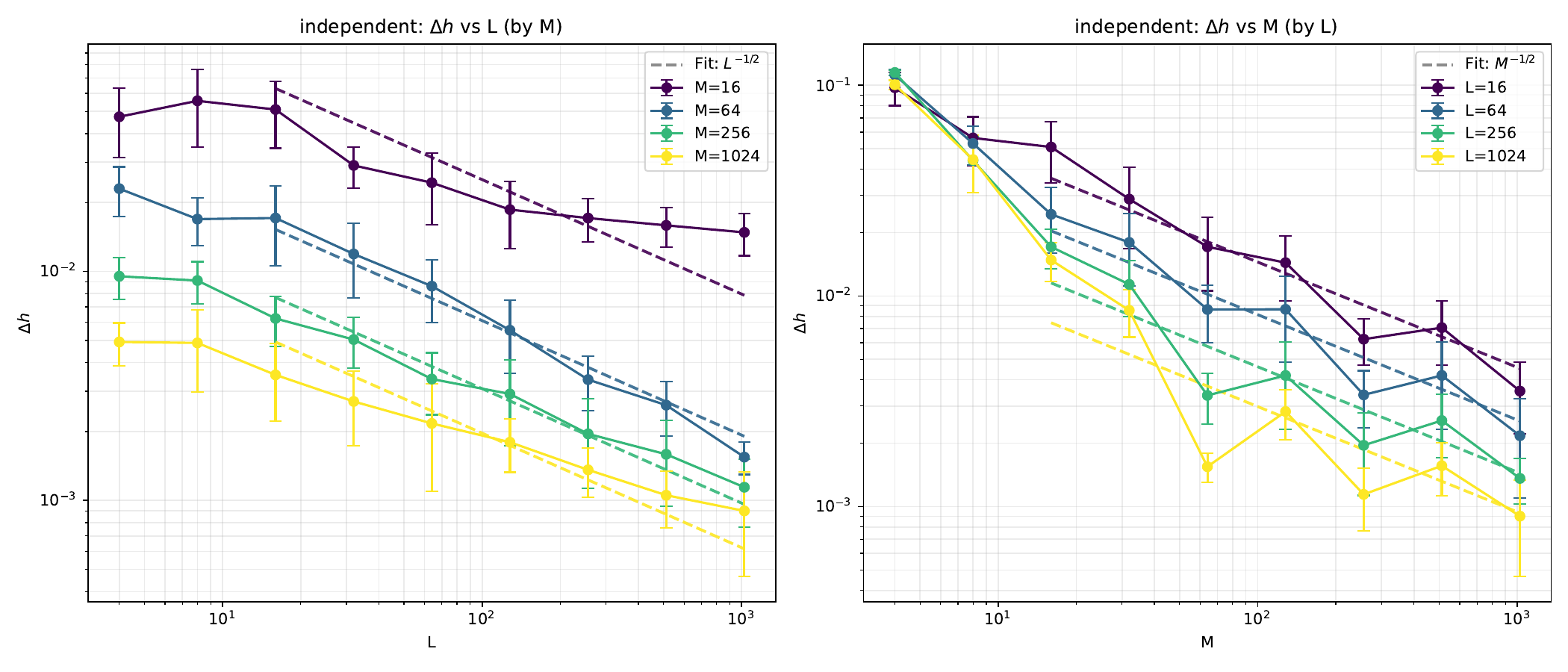}
        \caption{$\Delta^h_k$ for the \emph{independent} mask variant}
        \end{subfigure}\\
        \begin{subfigure}{\linewidth}
        \centering
        \includegraphics[width=0.8\linewidth]{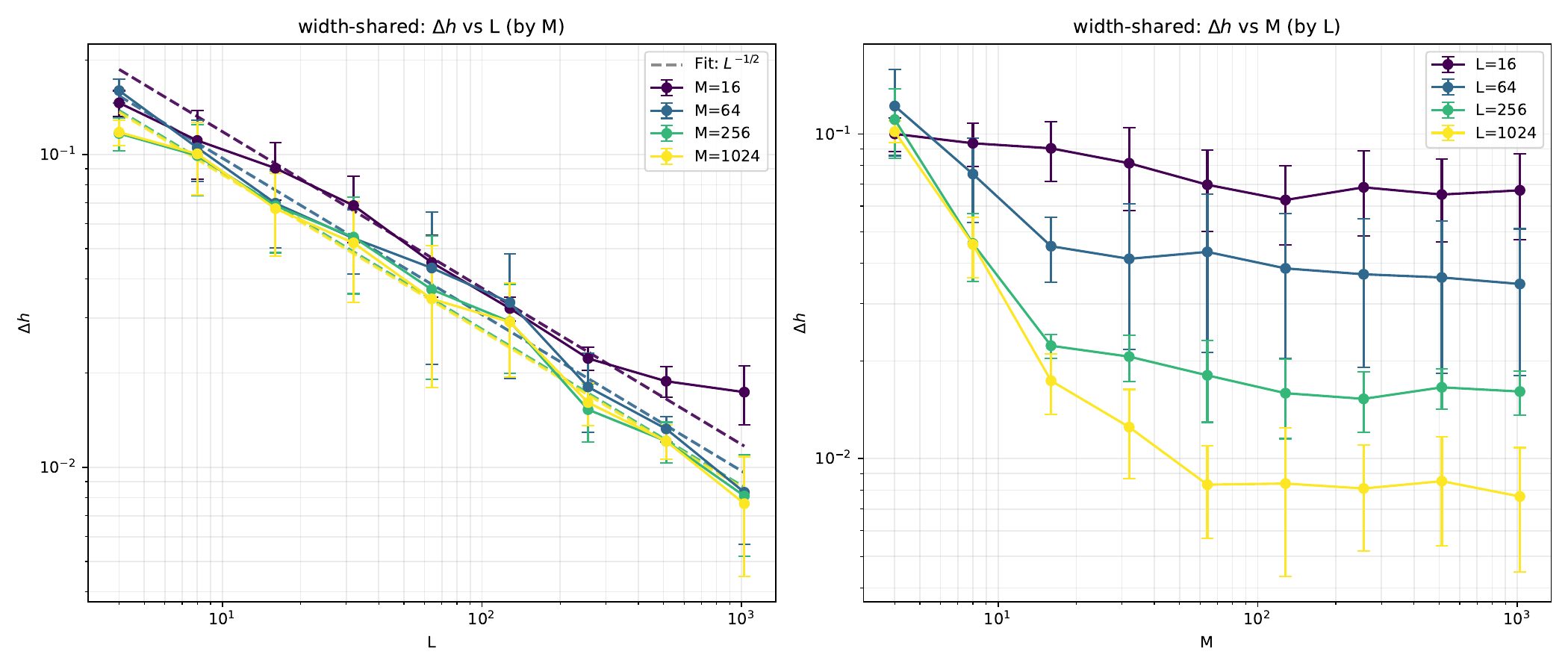}
        \caption{$\Delta^h_k$ for the \emph{width-shared} mask variant}
        \end{subfigure}\\
        \begin{subfigure}{\linewidth}
        \centering
        \includegraphics[width=0.8\linewidth]{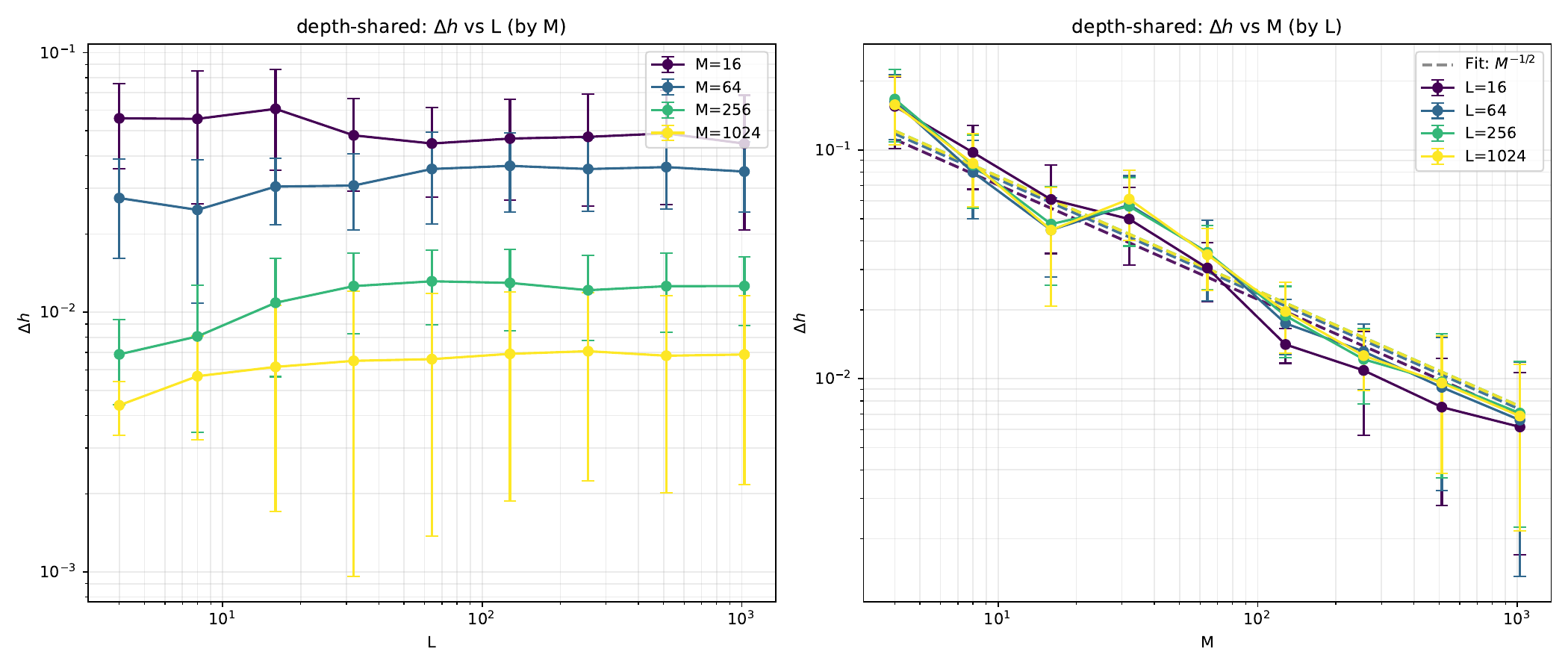}
        \caption{$\Delta^h_k$ for the \emph{depth-shared} mask variant}
        \end{subfigure}

        \caption{RMS error in the forward pass $\Delta^h_k$ between RaM and Dropout for the different masking strategies, in terms of $M$ (for different values of $L$) and $L$ (for different values of $M$), at training iteration $k=50$. We also fit the rates from Theorem~\ref{thm:main}.}
  \label{fig:details_forward_error_mnist}
\end{figure}

\begin{figure}
        \centering
        \begin{subfigure}{\linewidth}
        \centering
        \includegraphics[width=0.8\linewidth]{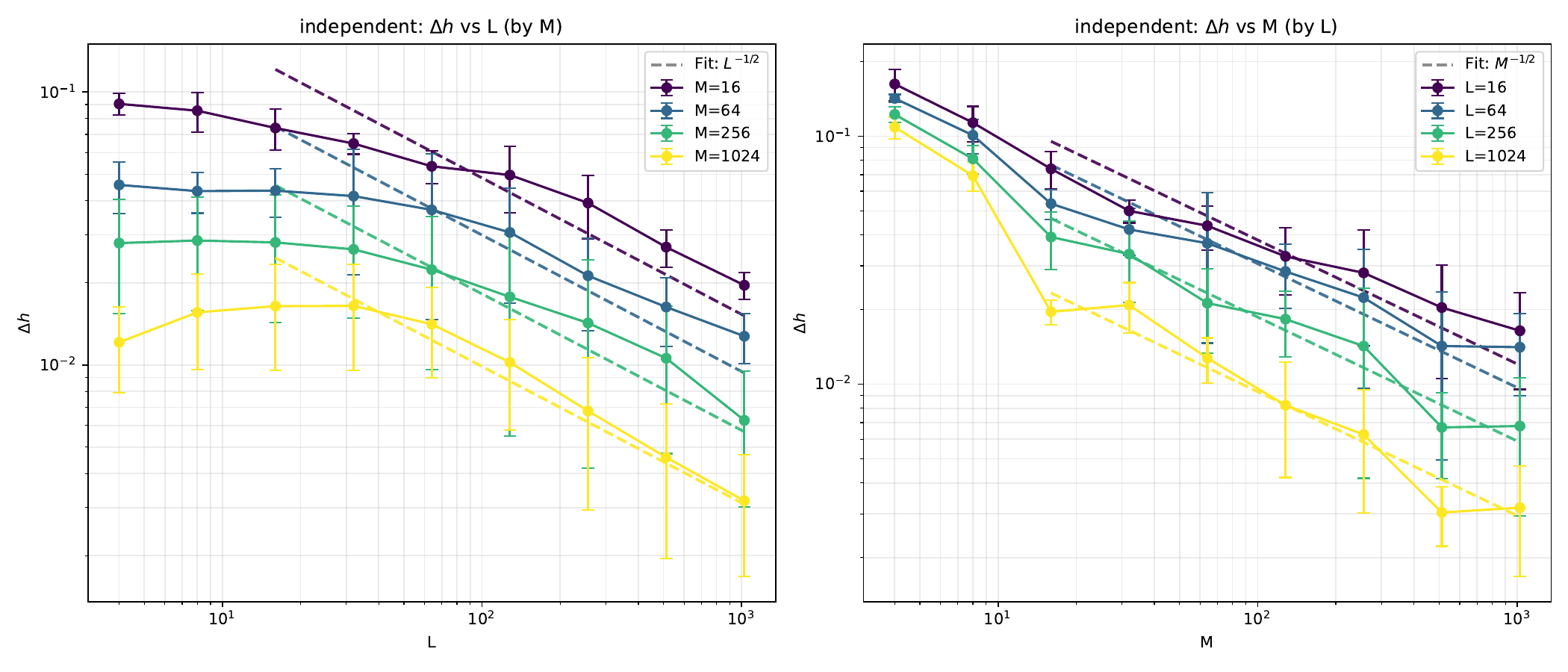}
        \caption{$\Delta^h_k$ for the \emph{independent} mask variant}
        \end{subfigure}\\
        \begin{subfigure}{\linewidth}
        \centering
        \includegraphics[width=0.8\linewidth]{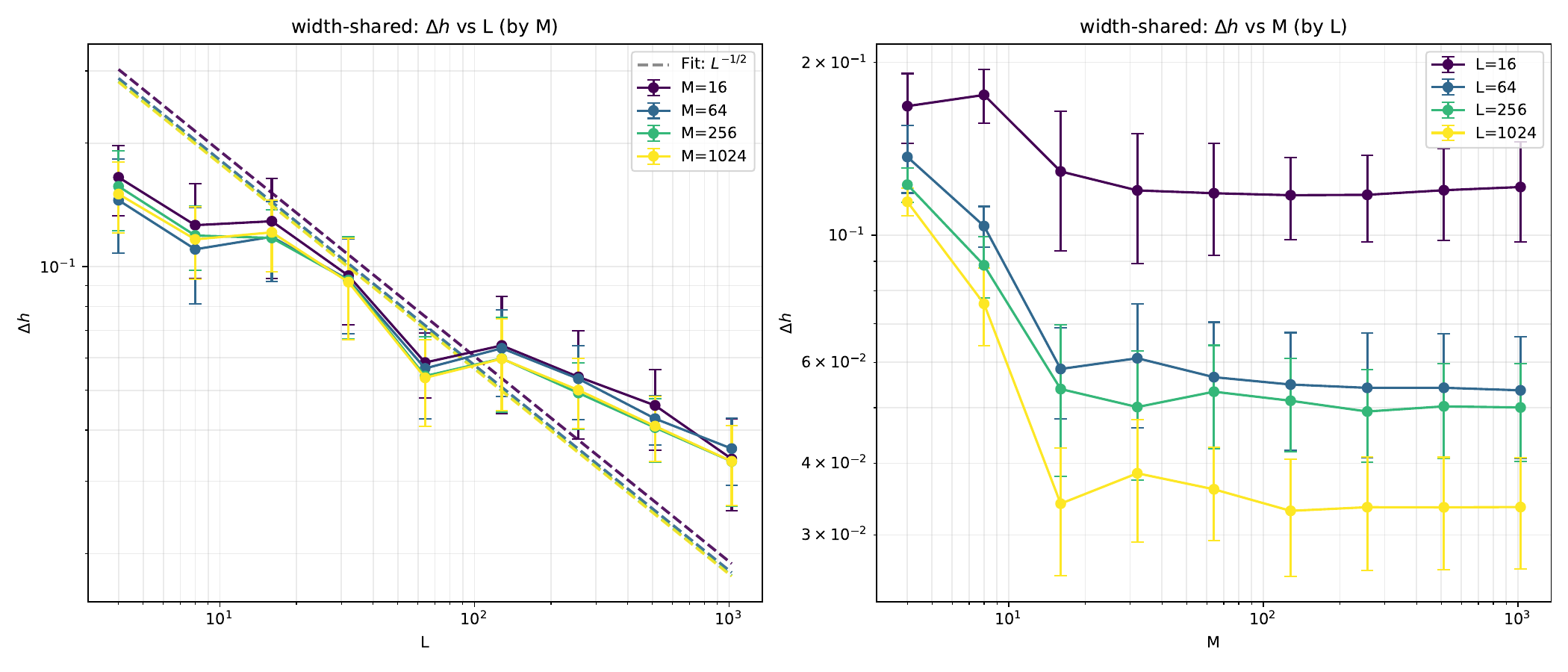}
        \caption{$\Delta^h_k$ for the \emph{width-shared} mask variant}
        \end{subfigure}\\
        \begin{subfigure}{\linewidth}
        \centering
        \includegraphics[width=0.8\linewidth]{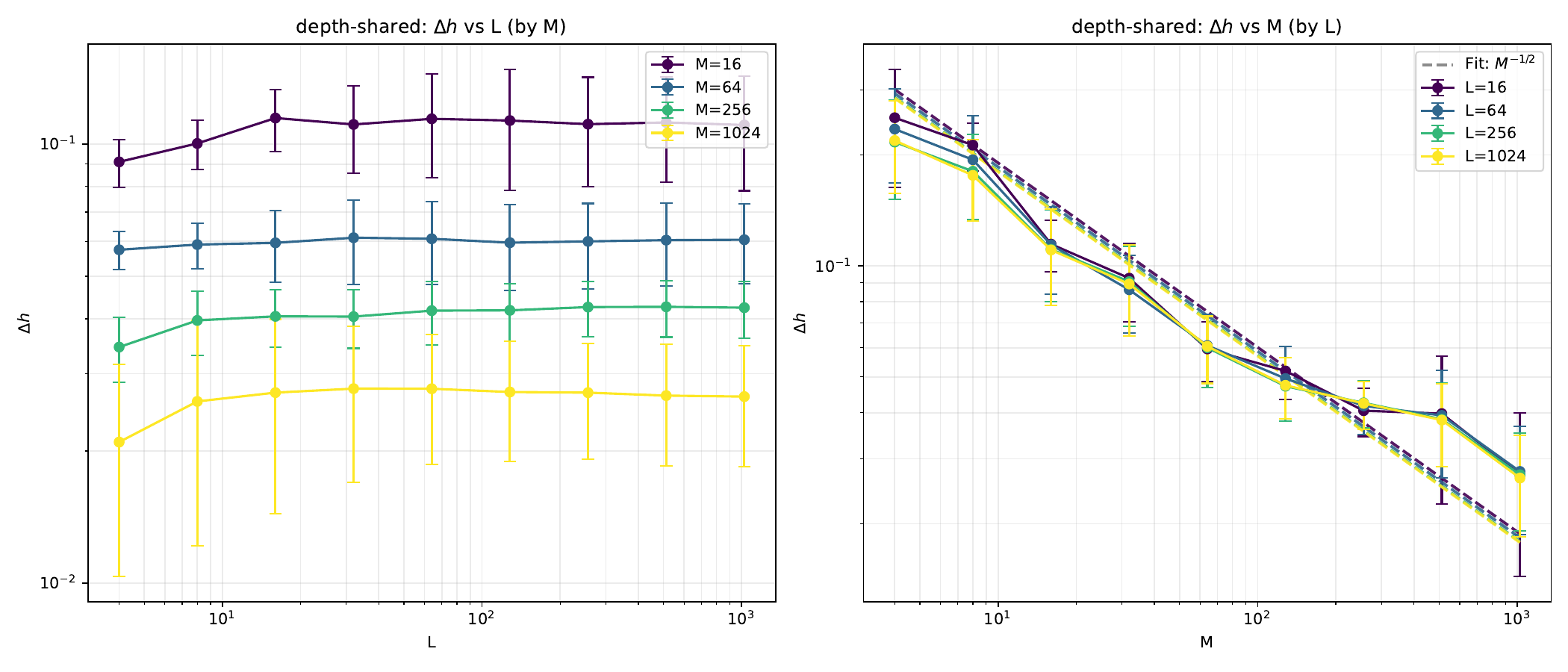}
        \caption{$\Delta^h_k$ for the \emph{depth-shared} mask variant}
        \end{subfigure}

        \caption{
        RMS forward-pass error \(\Delta_k^h\) between RaM and dropout for the
different masking strategies, plotted as a function of \(M\) for several
values of \(L\), and as a function of \(L\) for several values of \(M\),
at the final training iteration \(k=200\). The results show that our
analysis continues to capture the qualitative behavior after \(200\)
iterations, although the agreement with the rates predicted by
Theorem~\ref{thm:main} is less pronounced.}
  \label{fig:details_forward_error_mnist_k200}
\end{figure}

\end{document}